\documentclass[a4paper]{scrartcl}

\usepackage[utf8]{inputenc}
\usepackage[T1]{fontenc}
\usepackage[english]{babel}
\usepackage{amsmath,amssymb,amsfonts,siunitx,commath}
\usepackage{amsthm}
\usepackage{tikz,url}
\usepackage{geometry}
\usepackage{mathtools}
\mathtoolsset{showonlyrefs}
\usepackage{subcaption}
\usepackage{algorithm}
\usepackage{algpseudocode}
\usepackage{hyperref}
\usepackage{enumerate}
\usepackage{listings}

\newcommand{\R}{\mathbb{R}}

\newcommand{\E}{\mathbb{E}}
\newcommand{\N}{\mathbb{N}}

\newcommand{\dx}{\mathrm{d}}
\newcommand{\tT}{\mathrm{T}}

\theoremstyle{plain}
\newtheorem{lemma}{Lemma}
\newtheorem{theorem}[lemma]{Theorem}
\newtheorem{corollary}[lemma]{Corollary}

\newtheorem{remark}[lemma]{Remark}

\theoremstyle{definition}

\begin{document}
\title{Stochastic Normalizing Flows for Inverse Problems:
  a Markov Chains Viewpoint}
	\author{
Paul Hagemann\footnotemark[1]
\and
Johannes Hertrich\footnotemark[1]
\and
Gabriele Steidl\footnotemark[1]
}
\date{\today}

\maketitle

\footnotetext[1]{
TU Berlin,
Stra{\ss}e des 17. Juni 136, 
D-10587 Berlin, Germany,
\{hagemann,j.hertrich, steidl\}@math.tu-berlin.de.
} 

\begin{abstract}
To overcome topological constraints
and improve the expressiveness of normalizing flow architectures,
Wu, K\"ohler and No\'e introduced stochastic normalizing flows
which combine deterministic, learnable flow transformations 
with stochastic sampling methods. In this paper, we consider
stochastic normalizing flows from a Markov chain point of view.
In particular, we replace transition densities by general Markov kernels and establish proofs 
via Radon-Nikodym derivatives which allows to incorporate distributions without densities in a sound way. 
Further, we generalize the results for sampling from posterior distributions as required in inverse problems. 
The performance of the proposed  
conditional stochastic normalizing flow is demonstrated by numerical examples.
\end{abstract}

\section{Introduction} \label{sec:intro}
Deep generative models for approximating complicated and often high-dimensional probability distributions 
became a rapidly developing research field. 
Normalizing flows are a popular subclass of these generative models.
They can be used to model a target distribution  
by a simpler latent distribution which is usually the standard normal distribution. 
In this paper, we are interested in finite normalizing flows which are
basically concatenations of learned diffeomorphisms. The parameters of the diffeomorphism are adapted to the
target distribution by minimizing a loss functions.
To this end,  the diffeomorphism must have a tractable Jacobian determinant.
For the continuous counterpart of normalizing flows, we refer to the overview paper \cite{RH2021} and the references therein. 
Suitable architectures of finite normalizing flows include invertible
residual neural networks (ResNets) \cite{BGCDJ2019,CBDJ2019,HZRS2016},
(coupling-based) invertible neural networks (INNs) \cite{AKRK2019,DSB2017,kingma2018glow,MMRGN2018, RM2015} and
autoregessive flows \cite{CTA2019,DBMP2019,huang2018neural,PPM2017}.

Unfortunately, INNs as well as ResNets suffer from a limited expressiveness.
More precisely, their major drawbacks  are topological constraints, 
see, e.g. \cite{Falorsietal2018,FHDF2019}. 
For example, when trying to map a unimodal (Gaussian) distribution
to a multimodal one connections between the modes remain.
It was shown in \cite{HN2021}, see also \cite{behrmann2020understanding,CCDD19} 
that for an accurate match, the Lipschitz constant 
of the inverse flow has to go to infinity. Similar difficulties appear when mapping
to heavy-tailed distributions \cite{JKYB20}.
A special mixture model for the latent variable with sophisticated,
learnable probabilities depending on the ovservations was proposed in \cite{HN2021}.
In \cite{WKN2020}, Wu, K\"ohler and No\'e introduced so-called
stochastic normalizing flows (SNFs) consisting of a sequence of deterministic
flow transformation and stochastic sampling methods
with tractable paths, such as Markov Chain Monte Carlo (MCMC) \cite{robertsrosenthalMCMC} or overdamped Langevin
dynamics \cite{langevin_welling}. This is in very similar fashion to \cite{SWMG2015} where stochastic layers were used by learning diffusion kernels. Interestingly, they also establish a forward and backward trajectory so that the paper \cite{WKN2020} can be seen as a continuation of it.
Furthermore, flows combined with stochastic layers where also used in \cite{Chenetal2017, anicemc}.

Stochastic normalizing flows are closely related to the so-called nonequilibrium candidate Monte Carlo  method from nonequilibrium statistical mechanics introduced in \cite{NCMC2011}.
Here, the authors constructed a MCMC method by  generating a sequence $(x_n)_n$  by the following two steps:
first, based on the point $x_n$, they  construct a candidate $x'$ by a sequence of
deterministic flow transformations and stochastic sampling methods.
Second, they either accept or reject the point $x'$. If $x'$ is accepted, then 
$x_{n+1}\coloneqq x'$.
Otherwise, $x_{n+1}$ is set to a point $\tilde x'$ generated by the so-called
momentum reversed transformations of $x'$.
The first of these steps is very similar to SNFs with the difference that the 
deterministic flow transformations are not learned, but given by a certain application.
Furthermore, in \cite{arbel2021annealed} the authors propose the use of importance sampling and MCMC kernels in conjunction with normalizing flows, but in contrast to \cite{WKN2020} the layers are learned individually.
Moreover, the authors of \cite{NJHWW2020} combine deterministic and non-deterministic steps for increasing
the expressivness of normalizing flow models.

The contribution of this paper is twofold:

\emph{First}, we derive SNFs from a Markov chain point of view which has the
following advantage.
The authors of \cite{WKN2020} assumed within their theoretical considerations that any transition within a SNF
admits a probability density function.
Unfortunately, this assumption is not fulfilled if the transition is 
defined via a deterministic flow or Metropolis-Hastings transitions.
In this paper, we define SNFs as pairs of Markov chains $((X_0,...,X_T),(Y_T,...,Y_0))$.

Then, we can replace the transition densities by general Markov kernels and prove the 
corresponding results via Radon-Nikodym derivatives.
Further, we use our formal definitions to show in Theorem~\ref{Opt_cond} that in a "perfectly trained" SNF 
the distributions $P_{X_{t-1}}$ and $P_{X_t}$ before and after 
a MCMC layer are given by the stationary distribution ("equilibrium") of the corresponding MCMC kernel.

\emph{Second}, we extend the approach of SNFs to inverse problems, 
where we are interested in the posterior distribution
given the noisy output of an operator. 
For overview papers on deep learning in
inverse problems we refer to \cite{AMOS2019,OJMBDW2020} and
for a recent paper on using pre-trained flows for learning of conditional flow models to \cite{WLD2021}.
To sample from the posterior distribution we establish
conditional SNFs.
This generalizes in particular the INN approaches \cite{AFHHSS2021,ALKRK2019,KDSK2020}
for sampling from posterior distributions 
by incorporating stochastic layers.

The rest of the paper is organized as follows:
in Section~\ref{sec:MC}, we recall the Markov chain notation and consider normalizing flows within this setting.
Then, in Section~\ref{sec:SNF}, we introduce SNFs as pairs of Markov chains with
certain properties which are fulfilled for many approaches.
We propose and examine a corresponding loss function relaying on the Kullback-Leibler divergence.
The stochastic  Markov kernels/layers
are explained and a unbiased estimator for the gradient of the loss function is derived in Section~\ref{sec:stlay}.
So far no operators were involved in our setting.
This changes with Section~\ref{sec:inverse}, 
where we examine posterior estimations
related to inverse problems. To this end, we have to enlarge the definition of Markov chains
leading to conditional SNFs.
Section~\ref{sec:numerics}
demonstrates the performance of our conditional 
SNFs by an artificial example, where the ground truth can be computed analytically and 
a real-world application from scatterometry.
The code for the numerical examples is availible online\footnote{\url{https://github.com/PaulLyonel/conditionalSNF}}.
Finally, conclusions are drawn in Section~\ref{sec:conclusions}.

\section{Markov Chains and Normalizing Flows} \label{sec:MC}
In this section, we recall the basic notations of Markov chains and  normalizing flows
and relate both concepts. For an overview on Markov kernels see, e.g., \cite{LeGall}.

Let $(\Omega,\mathcal A, \mathbb P)$ be a probability space.
By a  probability measure on $\R^d$ we always mean a 
probability measure defined on the Borel $\sigma$-algebra $\mathcal B(\R^d)$.
Let ${\mathcal P}(\mathbb R^d)$ denote the set of probability measures on  $\mathbb R^d$.
Given a random variable $X: \Omega \rightarrow \R^d$, 
we use the push-forward notation $P_X = X_{\#} \mathbb P \coloneqq \mathbb P \circ X^{-1}$ for the corresponding  measure on $\R^d$.
A \emph{Markov kernel} 
$\mathcal K\colon \R^d\times \mathcal B(\R^d)\to [0,1]$ is a mapping such that
\begin{itemize}
\item[i)]  $\mathcal K(\cdot,B)$ is measurable for any $B\in\mathcal B(\R^d)$, and 
\item[ii)] $\mathcal K(x,\cdot)$ is a probability measure for any $x\in\R^d$.
\end{itemize}
For a probability measure $\mu$ on $\R^d$, the measure $\mu\times \mathcal K$ on 
$\R^d\times\R^d$ is defined by
\begin{equation*} \label{def}
(\mu\times \mathcal K)(A\times B)\coloneqq \int_A \mathcal K(x,B) \dx \mu(x)
\end{equation*}
and the measure $\mathcal K\times\mu$ by
\begin{equation} \label{def1}
(\mathcal K\times\mu)(A\times B) \coloneqq (\mu\times \mathcal K) (B\times A).
\end{equation}
Then, it holds for all integrable $f$ that
$$
\int_{\R^d\times\R^d}f(x,y) \dx (\mu\times \mathcal K)(x,y)=\int_{\R^d}\int_{\R^d}f(x,y) \dx \mathcal K(x,\cdot)(y) \dx \mu(x),
$$
and in particular, for $\mathbf A\subseteq \R^d\times\R^d$,
$$
(\mu\times\mathcal K)(\mathbf A)=\int_{\R^d}\int_{\R^d}1_{\mathbf A}(x,y) \dx \mathcal K(x,\cdot)(y) \dx \mu(x).
$$
A sequence of random variables $(X_0,\ldots,X_T)$, $T \in \mathbb N$ is called a \emph{Markov chain},
if there exist Markov kernels, also known as \emph{transition kernel},  
$\mathcal K_t\colon \R^d\times \mathcal B(\R^d)\to [0,1]$ 
which are versions of 
$P_{X_t|X_{t-1} = \cdot} (\cdot)$, $t=1,\ldots,T$, 
such that
\begin{align}\label{eq_path_measure}
P_{(X_0,...,X_T)} = P_{X_0} \times \mathcal K_{1} \times \cdots \times \mathcal K_{T}.
\end{align}
Note that we use the notion of the \emph{regular conditional distribution} of a random variable $X$ 
given a random variable $Y$ which is defined as the $P_Y$-almost 
surely unique Markov kernel $P_{Y|X}=P_{Y|X=\cdot}(\cdot)$ with the property
$$
P_{X}\times P_{Y|X}=P_{X,Y}.
$$
In this sense we will write 
$\mathcal K_t = P_{X_t|X_{t-1}}$ 
in \eqref{eq_path_measure}. 
A countable sequence $(X_t)_{t\in\N}$ is a \emph{Markov chain}, if \eqref{eq_path_measure} is fulfilled for every $T\in\N$.
If $(X_0,\ldots,X_T)$ is a Markov chain, then it can be shown that $(X_T,\ldots,X_0)$ is a Markov chain as well.

\subsection{Normalizing Flows as Markov Chains} \label{sec:NF}
A \emph{normalizing flow} is often understood as deterministic, invertible transform, 
which we call $\mathcal T_\theta\colon\R^d\to\R^d$, see \cite{pmlr-v37-rezende15}.
Here we focus on invertible neural network $\mathcal T_\theta$
which are briefly explained in the Appendix~\ref{app:INN}.
For better readability, we likewise skip the dependence of $\mathcal T_\theta$ on the parameter $\theta$ and write just
$\mathcal T = \mathcal T_\theta$.
Normalizing flows can be used
to model the target density $p_X$ of a distribution $P_X$ by a simpler latent distribution $P_Z$ 
which is usually the standard normal distribution.
This is done by learning $\mathcal T$ such that it holds 
\begin{equation} \label{eq:approx}
P_X \approx\mathcal T_\#P_Z,\quad\text{or equivalently}\quad P_Z\approx\mathcal T^{-1}_\#P_X.
\end{equation}
Note that we have by the change of variable formula for the corresponding densities 
\begin{equation} \label{push-forward}
p_{\mathcal T_\#P_Z} (x) = p_{Z} \big( \mathcal T^{-1} (x) \big) |\mathrm{det} \nabla \mathcal T^{-1} (x) |,
\end{equation}
and by the inverse function theorem  that
$
\Big(|\mathrm{det} \nabla \mathcal T_t^{-1} \left( \mathcal T_t(x_{t-1}) \right)| \big)^{-1} 
= |\mathrm{det}  \nabla \mathcal T_t(x_{t-1})|.
$
The approximation can be done by minimizing the Kullback-Leibler divergence.
Recall that the \emph{Kullback-Leibler divergence} 
$\mathrm{KL}\colon {\mathcal P}(\mathbb R^d) \times {\mathcal P}(\mathbb R^d) \rightarrow \mathbb [0, +\infty]$
of two measures
$\mu,\nu\in {\mathcal P}(\mathbb R^d)$ with existing Radon-Nikodym derivative 
$\frac{\dx \mu}{\dx \nu}$ of $\mu$ with respect to $\nu$ is defined by 
\begin{equation} \label{KLdef}
\mathrm{KL} (\mu,\nu) \coloneqq \int_{\mathbb R^d} \log \Big(\frac{\dx \mu}{\dx \nu} \Big) \, \dx \mu
 = \E_{x \sim \mu} \Big[\log \Big( \frac{\dx \mu}{\dx \nu}\Big) \Big]. 
\end{equation}
In case that the above Radon-Nikodym derivative does not exist, we have $\mathrm{KL} (\mu,\nu) = + \infty$.
Then we have
\begin{align*}
\mathrm{KL}(P_X,\mathcal T_\#P_Z)
&=
\E_{x\sim P_X}\Big[\log\Big( \frac{ p_X}{ p_{\mathcal T_\#P_Z}} \Big)\Big]
= 
\E_{x\sim P_X}\left[\log p_X \right]-\E_{x\sim P_X} \left[\log p_{\mathcal T_\#P_Z} \right]
\\
&=
\E_{x\sim P_X}\left[\log p_X \right]-
\E_{x\sim P_X} \left[\log  p_Z \circ \mathcal T^{-1}  \right] 
 - 
\E_{x\sim P_X} \left[\log |\mathrm{det}(\nabla \mathcal T^{-1} )| \right].
\end{align*}
Since the first summand is just a constant, this gives rise to the loss function
\begin{equation} \label{loss_n}
\mathcal L_{\text{NF}} (\theta) \coloneqq -
\E_{x\sim P_X} \left[\log  p_Z \circ \mathcal T^{-1}  \right] - 
\E_{x\sim P_X} \left[\log |\mathrm{det}(\nabla \mathcal T^{-1} )| \right].
\end{equation}
The network $\mathcal T$ is constructed by concatenating smaller blocks 
$$\mathcal T=\mathcal T_T\circ\cdots\circ \mathcal T_1$$
which are invertible networks on their own.
Then, the blocks $\mathcal T_t\colon\R^d\to\R^d$ 
generate a pair of Markov chains $\big((X_0,...,X_T),(Y_T,...,Y_0)\big)$ by 
\begin{align*}
X_0\sim P_Z,\quad &X_t=\mathcal T_t(X_{t-1}),\\
Y_T\sim P_X, \quad &Y_{t-1}=\mathcal T_t^{-1}(Y_t)
\end{align*}
with corresponding Markov kernels
\begin{equation} \label{kern_det}
\mathcal K_t(x,\cdot)= P_{X_t|X_{t-1}} = \delta_{\mathcal T_t(x)},\qquad
\mathcal R_t(x,\cdot)= P_{Y_{t-1}|Y_{t}} = \delta_{\mathcal T_t^{-1}(x)},
\end{equation}
where
$$
\delta_x(A)\coloneqq\begin{cases}1,&$if $x\in A,\\0,&$otherwise.$\end{cases}
$$
Due to their correspondence to the layers $\mathcal T_t$ and $\mathcal T_t^{-1}$ 
from the normalizing flow $\mathcal T$, 
we call the Markov kernels 
$\mathcal K_t$ \emph{forward layers}, 
while the Markov kernels $\mathcal R_t$ 
are called \emph{reverse layers}. 
Both are so-called \emph{deterministic layers}.

\section{Stochastic Normalizing Flows} \label{sec:SNF}
The limited expressiveness of normalizing flows
can be circumvented by introducing so-called stochastic normalizing flows (SNFs).
In our Markov chain notation this means that some of the above deterministic layers
are replaced by stochastic ones. 

A \emph{stochastic normalizing flow} (SNF) is a pair  
$\big( (X_0,\ldots,X_T),(Y_T,...,Y_0) \big)$ of Markov chains of $d$-dimensional random variables $X_t$ and $Y_t$, $t=0,...,T$,
with the following properties:
\begin{itemize}
\item[P1)]  $P_{X_t}, P_{Y_t}$ have the densities $p_{X_t}, p_{Y_t}\colon\R^{d}\to\R_{>0}$ for any $t=0,...,T$.
\item[P2)]  There exist Markov kernels 
$\mathcal K_{t} = P_{X_t|X_{t-1}}$ and  
$\mathcal R_t = P_{Y_{t-1}|Y_t}$, $t=1,...,T$
such that 
\begin{align*}
P_{(X_0,...,X_T)} &= P_{X_0} \times P_{X_1|X_{0}}  \times \cdots \times P_{X_T|X_{T-1}},\\ 
P_{(Y_T,...,Y_0)} &= P_{Y_T} \times P_{Y_{T-1}|Y_{T}} \times\cdots\times P_{Y_{0}|Y_{1}}.
\end{align*}
\item[P3)] For $P_{X_t}$-almost every $x\in\R^{d_t}$, 
the measures $P_{Y_{t-1}|Y_t=x}$ and $P_{X_{t-1}|X_t=x}$ 
are absolutely continuous with respect to each other.
\end{itemize}
SNFs were initially introduced in \cite{WKN2020}, but the above definition via Markov chains is novel.
Clearly, deterministic normalizing flows are special cases of a SNFs.
In our applications, Markov chains usually start with a latent random variable 
$X_0 = Z$,
which is easy to sample from, 
and we intend  to learn the Markov chain such that $X_T$ 
approximates a target random variable $X$, while
the random variable $Y_T$ is initialized by
$Y_T = X$ 
from a data space and $Y_0$ should approximate the
latent variable $Z$.
Once, we have learned the SNF $((X_0,...,X_T),(Y_T,...,Y_0))$, we can sample from the approximation $X_T$ of $X$ as follows:
\begin{itemize}
\item[-] Draw a sample $x_0$ from $X_0=Z$.
\item[-] For $t=1,...,T$, draw samples $x_t$ from $\mathcal K_t(x_{t-1},\cdot)=P_{X_t|X_{t-1}=x_{t-1}}$.
\end{itemize}
Then, the samples $x_T$ generated by this procedure follow the distribution $P_{X_T}$.

Unfortunately, for stochastic layers it is not known how to 
minimize 
$\mathrm{KL}(P_X,P_{X_T})$ as it was done for normalizing flows.
Instead, we minimize the KL divergence of the joint distributions
\begin{equation}\label{KL-loss}
\mathcal L_\text{SNF}(\theta)\coloneqq \mathrm{KL}(P_{(Y_0,...,Y_T)},P_{(X_0,...,X_T)}).
\end{equation}
It can be shown that \eqref{KL-loss} an upper bound of 
$\mathrm{KL}(P_{Y_T},P_{X_T})=\mathrm{KL}(P_X,P_{X_T})$.
In the case of normalizing flows we will see that the expressions coincides,
i.e. $\mathcal L_\text{SNF} = \mathcal L_\text{NF}$ up to a constant.
For its minimization we need the following lemma. Since we could not find a reference, we give the
proof in Appendix~\ref{app:proofs}.

\begin{lemma}\label{prop_KL_MC_paths}
Let $((X_0,...,X_T),(Y_T,...,Y_0))$ be a SNF.
Then, the Radon-Nikodym derivative $f \coloneqq \frac{\dx P_{(Y_0,...,Y_T)}}{\dx P_{(X_0,...,X_T)}}$ 
is given by 
\begin{align}
f(x_0,...,x_T)
&=\frac{p_{Y_0}(x_0)}{p_{X_0}(x_0)}\prod_{t=1}^T g_t(x_{t},x_{t-1}) 
=
\frac{p_{Y_T}(x_T)}{p_{X_T}(x_T)}\prod_{t=1}^T f_t(x_{t-1},x_{t}) \nonumber\\
&=
\frac{p_{Y_T}(x_T)}{p_{X_0}(x_0)}\prod_{t=1}^T \frac{f_t(x_{t-1},x_t)p_{X_{t-1}}(x_{t-1})}{p_{X_t}(x_t)}.
\label{def_f}
\end{align}
with Radon-Nikodym derivatives
$$
g_t(\cdot,x_{t-1}) \coloneqq \frac{\dx P_{Y_{t}|Y_{t-1}=x_{t-1}}}{\dx P_{X_{t}|X_{t-1}=x_{t-1}}}, \quad
f_t(\cdot,x_{t}) \coloneqq \frac{\dx P_{Y_{t-1}|Y_{t}=x_{t}}}{\dx P_{X_{t-1}|X_{t}=x_{t}}}.
$$
\end{lemma}

Since by definition 
$$
\mathrm{KL}(P_{(Y_0,...,Y_T)},P_{(X_0,...,X_T)})=\E_{(x_0,...,x_T)\sim P_{(Y_0,...,Y_T)}}[\log(f(x_0,...,x_T))],
$$
the lemma implies immediately the following theorem.
%
\begin{theorem}\label{thm_KL_paths}
Let $((X_0,...,X_T),(Y_T,...,Y_0))$ be a SNF
and $X_0=Z$, $Y_T = X$.
Then, loss function in \eqref{KL-loss} is given by
\begin{align*}
&\quad \mathcal L_\mathrm{SNF}(\theta)=\mathrm{KL}(P_{(Y_0,...,Y_T)},P_{(X_0,...,X_T)})
\\
&=
\E_{(x_0,...,x_T)\sim P_{(Y_0,...,Y_T)}}
\Big[-\log(p_{X_T}(x_T))+\log(p_X(x_T))+\sum_{t=1}^T \log(f_t(x_{t-1},x_t))\Big]
\\
&=\E_{(x_0,...,x_T)\sim P_{(Y_0,...,Y_T)}}
\Big[-\log(p_Z(x_0))+\log(p_X(x_T))+\sum_{t=1}^T
\log\Big(\frac{f_t(x_{t-1},x_t) p_{X_{t-1}(x_{t-1})}}{p_{X_{t}(x_{t})}}\Big)
\Big].
\end{align*}
\end{theorem}

\begin{remark} \label{rem:wrong}
If \emph{all conditional distributions} 
$P_{X_t|X_{t-1}}$ \emph{and} $P_{Y_{t-1}|Y_{t}}$, $t=1,\ldots,T$ as well as the distributions $P_{X_0}$ and $P_{Y_T}$
are \emph{absolutely continuous with respect to the Lebesgue measure}, i.e.
have positive densities $p_{X_t|X_{t-1}=x_{t-1}}$, $p_{Y_{t-1}|Y_{t}=x_{t}}$, $p_{X_0}$ and $p_{Y_T}$,
then the distribution of $(X_0,...,X_T)$ has the density function
$$
p_{(X_0,...,X_T)}(x_0,...,x_T)=p_{X_0}(x_0)\prod_{t=1}^T p_{X_t|X_{t-1}=x_{t-1}}(x_t)
$$
and similarly for the distributions of $(Y_0,...,Y_T)$.
In this special case, the Radon-Nikodym derivative in \eqref{def_f}
becomes the quotient of the corresponding density functions
\begin{equation} \label{desire}
f(x_0,...,x_T)=\frac{p_{(X_0,...,X_T)}(x_0,...,x_T)}{p_{(Y_0,...,Y_T)}(x_0,...,x_T)}
=\frac{p_{X_0}(x_0)}{p_{Y_T}(x_T)}
\prod_{t=1}^T\frac{p_{X_t|X_{t-1}=x_{t-1}}(x_t)}{p_{Y_{t-1}|Y_{t}=x_{t}}(x_{t-1})}.
\end{equation}
This form is used in several papers as definition of the Radon-Nikodym derivative of 
$P_{(X_0,...,X_T)}$ with respect to $P_{(Y_0,...,Y_T)}$
and the densities $p_{X_t|X_{t-1}=x_{t-1}}(x_t)$ and $p_{Y_{t-1}|Y_{t}=x_{t}}(x_{t-1})$ 
are often called \emph{forward} and \emph{backward probabilities}.
However, if some conditional distributions 
do not have a density, then \eqref{desire} is not longer well-defined.
For instance, in the deterministic case $X_t=\mathcal T_t(X_{t-1})$, 
we have
$P_{X_t|X_{t-1}=x}=\delta_{\mathcal T_t(x)}$ and $P_{Y_{t-1}|Y_t=x}=\delta_{\mathcal T_t^{-1}(x)}$.
Later, we will consider MCMC layers, which also do not have densities.
Then, e.g., the authors of \cite{WKN2020} 
replace the densities $p_{X_t|X_{t-1}=x_{t-1}}(x_t)$ and $p_{Y_{t-1}|Y_t=x_t}(x_{t-1})$ in \eqref{desire}
by the ``$\delta$-functions/distributions" $\delta(x_t-\mathcal T_t(x_{t-1}))$ and $\delta(x_{t-1}-\mathcal T_t^{-1}(x_t))$ and
invest some effort to handle the quotient 
$$\frac{\delta(x_t-\mathcal T_t(x_{t-1}))}{\delta(x_{t-1}-\mathcal T_t^{-1}(x_t))}.$$
From a mathematical perspective it is not even clear how this expression is defined, such that a mathematically rigorous treatment
is not longer possible.
In contrast, our approach involves the quotients in \eqref{def_f} which only require that
the $P_{Y_{t-1}|Y_t=x}$ and $P_{X_{t-1}|X_t=x}$ 
are \emph{absolutely continuous with respect to each other}, see P3) and that $P_{X_t}$ 
and $P_{Y_t}$ are absolutely continuous, see P1).
This condition is fulfilled for most of the existing approaches. \hfill $\square$
\end{remark}

\section{Stochastic MCMC and Langevin Layers} \label{sec:stlay}
Next, we consider two kind of stochastic layers, namely
MCMC and (overdamped) Langevin kernels with fixed parameters.
Both layers were also used in \cite{WKN2020}.
Here the corresponding Markov kernels and those in the reverse layers will coincide, i.e.,
$$\mathcal K_t = \mathcal R_t.$$

\subsection{Langevin Layer} 
%
By $\mathcal N(m,\Sigma)$, we denote the normal distribution on $\mathbb R^d$ with density
\begin{equation}\label{gaussian}
		\mathcal N (x|m,\Sigma) = (2\pi)^{-\frac{d }{2}} |\Sigma|^{-\frac{1}{2}} 
		\,\exp\left(-\frac{1}{2}(x-m)^\tT \Sigma^{-1}(x-m) \right).
	\end{equation} 
Let $\xi_t\sim \mathcal N(0,I)$ such that $\sigma(\xi_t)$ 
and $\sigma \left( \cup_{s\le t-1} \sigma( X_s) \right)$ are independent. 
Here $\sigma(X)$ denotes the smallest $\sigma$-al\-ge\-bra 
generated by the random variable $X$.
We assume that we are given a proposal density $p_t\colon\R^d\to\R_{>0}$ which we specify later.
We denote by $u_t(x)\coloneqq -\log(p_t(x))$ the negative log-likelihood of $p_t$
and set
$$
X_t \coloneqq X_{t-1}-a_1 \nabla u_t(X_{t-1})+a_2\xi_t,
$$
where $a_1,a_2>0$ are some predefined constants.
Then, the transition kernel is given by
\begin{align} \label{eq_langevin_kernel}
\mathcal K_t(x,\cdot) = \mathcal N(x-a_1\nabla u_t(x),a_2^2 I).
\end{align}

\subsection{MCMC Layer} 
Let $X_t'$ be a random variable and $U\sim\mathcal U_{[0,1]}$ such that
\\
$\left( \sigma(X_t'),\sigma(U),\sigma \left( \cup_{s\le t-2} \sigma( X_s) \right) \right)$ 
are independent. 
Further, we assume that the joint distribution $P_{X_{t-1},X_t'}$ is given by
$$
P_{X_{t-1},X_t'}=P_{X_{t-1}}\times Q_t
$$
for some appropriately chosen Markov kernel $Q_t\colon\R^d\times\mathcal B(\R^d)\to[0,1]$,
where $Q_t(x,\cdot)$ is assumed to have the strictly positive probability density function $q_t(\cdot|x)$.
Then, for a proposal density $p_t\colon\R^d\to\R_{>0}$ which we specify later, we set
\begin{align}
X_t 
&\coloneqq
1_{[U,1]} \left( \alpha_t( X_{t-1},X_t') \right) \, X_t'
+
1_{[0,U]}  \left( \alpha_t( X_{t-1}, X_t' ) \right) \, X_{t-1}
\end{align} 
where
$$\alpha_t (x,y) \coloneqq \min \left\{1, \frac{p_t(y)q_t(y|x)}{p_t(x)q_t(x|y)} \right\}.$$
The corresponding transition kernel 
$\mathcal K_t \colon \R^d\times \mathcal B(\R^d)\to [0,1]$ 
is given by
\begin{equation} \label{kernel_MCMC}
\mathcal K_t(x,A) 
\coloneqq
\int_A q_t (y|x) \alpha_t (x,y) \dx y 
+
\delta_x (A) \int_{\mathbb R^d} q_t(y|x) \left( 1- \alpha_t (x,y) \right) \dx y .
\end{equation}

\begin{remark}[Choice of $Q_t$]\label{rem_MCMC_step_kernels}
In our numerical experiments, we consider two choices of $Q_t$.
\begin{enumerate}[(i)]
    \item The first and most simple idea is to use
    $$
    Q_t(x,\cdot)=\mathcal N(x,\sigma^2 I),\qquad q(\cdot|x)=\mathcal N(\cdot|x,\sigma^2 I).
    $$
    In this case, we have that $X_t'=X_{t-1}+\xi_t$, where $\xi_t\sim\mathcal N(0,\sigma^2 I)$ such that
    $$
    X_t=X_{t-1}+1_{[U,1]} \left( \alpha_t( X_{t-1},X_t') \right) \, \xi_t
    $$
    \item The second choice of $Q_t$ is the kernel \eqref{eq_langevin_kernel} from the Langevin layer, i.e.,
    $$
    Q_t(x,\cdot)=\mathcal N(x-a_1\nabla u_t(x),a_2^2 I),\qquad q(\cdot|x)=\mathcal N(\cdot|x-a_1\nabla u_t(x),a_2^2 I).
    $$
    Then, we have that $X_t'=X_{t-1}-a_1\nabla u_t(X_{t-1})+a_2\xi_t$, where $\xi_t\sim\mathcal N(0,I)$ such that
    $$
    X_t=X_{t-1}+1_{[U,1]} \left( \alpha_t( X_{t-1},X_t') \right) \, (a_2\xi_t-a_1\nabla u_t(X_{t-1}))
    $$
\end{enumerate}
\end{remark}

\paragraph{Relation to Metropolis-Hastings algorithm}
Indeed, this transition kernel is the kernel of a simple MCMC algorithm, 
namely the Metropolis-Hastings algorithm, see e.g. \cite{robertsrosenthalMCMC}.
Let us briefly recall this algorithm to see the relation.
We aim to sample approximately from a proposal distribution 
$P$ on $\mathbb R^d$ with probability density function $p$,
where we can evaluate $p$ at points in $\mathbb R^d$.
For an appropriately chosen Markov kernel $Q\colon\R^d\times\mathcal B(\R^d)\to[0,1]$,
where $Q(x,\cdot)$ is assumed to have the strictly positive probability density function $q(\cdot|x)$
the \emph{Metropolis-Hastings algorithm} generates a sequence $(x_n)_{n\in\N}$ 
starting at $x_0 \in \mathbb R^d$ by the following steps.
\begin{enumerate}
\item Draw $x'$ from $Q(x_n,\cdot)$ and $u$ uniformly in $[0,1]$.
\item Compute the acceptance ratio
$$
\alpha(x_n,x') \coloneqq \min \left\{1, \frac{p(x')q(x_n|x')}{p(x_n)q(x'|x_n)} \right\}.
$$
\item Set
$$x_{n+1} \coloneqq 
\left\{
\begin{array}{ll}
x' & \mathrm{if} \;  u<\alpha(x',x_n),\\
x_{n+1} \coloneqq x_n & \mathrm{otherwise}.
\end{array}
\right.
$$
\end{enumerate}

The Metropolis-Hastings algorithm generates samples of a time-homogeneous Markov chain 
$(X_n)_{t\in\N}$ starting at $X_0$ 
with  Markov kernel $\mathcal K_{\mathrm{MH}} \colon\R^d\times\mathcal B(\R^d)\to[0,1]$ given by
\begin{align}\label{eq_MH_kernel}
\mathcal K_{\mathrm{MH}}(x,A) &= \int_A q(y|x) \alpha (x,y) \dx y 
+ \delta_x(A) \int_{\mathbb R^d} q(y|x) \left(1-\alpha(x,y) \right)\dx y.
\end{align}
Recall that A Markov chain is called \emph{time-homogeneous}, if $\mathcal K_{t} = \mathcal K_{t'}$ for all $t,t' \in \mathbb N$.
Under mild assumptions, the Markov chain $(X_n)_{n \in \mathbb N}$ admits the unique stationary distribution $P$ 
and $P_{X_n}\to P$ as $n \rightarrow \infty$ in the total variation norm, see, e.g. \cite{TLA2020}.
We will need Markov kernels $\mathcal K_{\mathrm{MH}}$
fulfilling a \emph{detailed balance condition with respect to} $P$, resp. $p$, i.e.
\begin{align}\label{eq_detailed_balance}
\int_A \mathcal K_{\mathrm{MH}}(x,B) \dx P(x)=\int_B \mathcal K_{\mathrm{MH}}(x,A) \dx P(x), 
\quad \text{for all } A,B\in\mathcal B(\R^d).
\end{align}
By \eqref{def1} the detailed balance condition can be reformulated as 
$P \times \mathcal K_{\mathrm{MH}}=\mathcal K_{\mathrm{MH}}\times P$.
It can be shown that the kernel $\mathcal K_{\mathrm{MH}}$ in \eqref{eq_MH_kernel} 
fulfills the detailed balance condition with respect to $P$ \cite{robertsrosenthalMCMC}.

Now it becomes clear that our transition kernel $\mathcal K_t$ in \eqref{kernel_MCMC} 
is a Metropolis-Hastings kernel with respect to $p=p_t$ and the chosen Markov kernels from Remark~\ref{rem_MCMC_step_kernels}. 
Clearly, we have for this setting that $\mathcal K_t$ fulfills the 
detailed balance condition with respect to $p_t$.

In the case, that $Q$ is given by the kernel from Remark~\ref{rem_MCMC_step_kernels} (ii), the Metropolis
Hastings algorithm is also called Metropolis-adjusted Langevin algorithm (MALA), see \cite{GC2011,RT1996}.

\begin{remark}[Interpolation of the target densities]	
Recall that we intend to sample from a random variable $X$ with given density $p_X$ using 
a random variable $Z$ with density $p_Z$ by a Markov chain $(X_0,\ldots,X_T)$ such that 
$X_0 = Z$
and 
$X_T$ approximates $X$ in a sense we have to specify. 
Therefore it appears reasonable to choose the target densities $p_t$ of the stochastic layers 
as a certain interpolation between $p_Z$ and $p_X$. In this paper, we use the geometric mean
$\smash{p_t = c \, p_Z^{(T-t)/T}p_{X}^{t/T}}$,
where $c$ is a normalizing constant.
For an interpretation of this geometric mean as weighted Riemannian center of mass between $p_Z$ and $p_X$ (pointwise evaluated) 
on the manifold of positive numbers $\mathbb R_{>0}$ with distance 
$\dx_{\mathbb R_{>0}} (q_1,q_2) \coloneqq |\log q_1 - \log q_1|$
we refer to \cite{BFPS17}.
\end{remark}

\subsection{Training of Stochastic Layers} \label{subsec:learn}
To learn SNFs we have to specify  the quotients 
$\smash{\tfrac{f_t(x_{t-1},x_t)p_{X_t}(x_t)}{p_{X_{t-1}}(x_{t-1})}}$
for the deterministic and stochastic layers in the loss function in Theorem~\ref{thm_KL_paths}.
This is done in the next theorem.

\begin{theorem}\label{lem_logdets}
Let $((X_0,...,X_T),(Y_T,...,Y_0))$ be a SNF
and $(x_0,...,x_T) \in \mathrm{supp}(P_{(X_0,...,X_T)})$
\\
$=\mathrm{supp}(P_{(Y_0,...,Y_T)})$. 
Let $f_t(\cdot,x_t)$ be the Radon-Nikodym derivative $\frac{\dx P_{Y_{t-1}|Y_t=x_t}}{\dx P_{X_{t-1}|X_t=x_t}}$.
Then the following holds true:
\begin{enumerate}
\item[\textrm{i)}] If $\mathcal K_t$ is a deterministic layer \eqref{kern_det} with 
for some diffeomorphism $\mathcal T_t\colon\R^d\to\R^d$ 
and \\
$\mathcal R_t(x,A) = \delta_{\mathcal T_t^{-1}(x)}(A)$, then
$$
\frac{p_{X_{t-1}}(x_{t-1})}{p_{X_{t}}(x_{t})}=\frac1{|\nabla \mathcal T_t^{-1}(x_{t})|}
\quad\text{and}\quad 
f_t(x_{t-1},x_t)=1.
$$
\item[\textrm{ii)}] If $\mathcal K_t$ fulfills the detailed balance condition \eqref{eq_detailed_balance} 
with respect to some density $p_t\colon\R^d\to\R_{>0}$ and  $\mathcal R_t = \mathcal K_t$,
then
$$
\frac{f(x_{t-1},x_t)p_{X_{t-1}}(x_{t-1})}{p_{X_{t}}(x_{t})}=\frac{p_t(x_{t-1})}{p_t(x_{t})}.
$$
\item[\textrm{iii)}] 
If $\mathcal K_t(x,\cdot) = P_{X_t|X_{t-1}=x}$ 
admits the density $p_{X_t|X_{t-1}=x}\colon\R^d\to\R_{>0}$ and $\mathcal R_t = \mathcal K_t$, then
\begin{align}\label{eq_logdet_density}
\frac{f(x_{t-1},x_t)p_{X_{t-1}}(x_{t-1})}{p_{X_{t}}(x_{t})}
=
\frac{p_{X_t|X_{t-1}=x_t}(x_{t-1})}{p_{X_t|X_{t-1}=x_{t-1}}(x_t)}.
\end{align}
Moreover, if $\mathcal K_t$ is the Langevin kernel \eqref{eq_langevin_kernel} 
with proposal density $p_t\colon\R^d\to\R_{>0}$ and $u_t=-\log(p_t)$, then
$$
\frac{f(x_{t-1},x_t)p_{X_{t-1}}(x_{t-1})}{p_{X_{t}}(x_{t})}
=
\exp\Big(\frac12(\|\eta_t\|^2-\|\tilde\eta_t\|^2)\Big),
$$
where
$$
\eta_t\coloneqq \frac1{a_2}\big(x_{t-1}-x_t-a_1\nabla u_t(x_{t-1})\big),
\quad \tilde\eta_t\coloneqq \frac1{a_2}\big(x_{t-1}-x_t+a_1\nabla u_t(x_{t})\big).
$$
\end{enumerate}
\end{theorem}

Note that case ii) includes in particular the MCMC layer.

\begin{proof}
i) Since $X_t=\mathcal T_t(X_{t-1})$ holds $P_{(X_0,...,X_T)}$-almost surely and since $(x_0,...,x_T)$ 
is contained in the support of $P_{(X_0,...,X_T)}$, we have that $x_t=\mathcal T_t(x_{t-1})$. 
Thus, the change of variables formula \eqref{push-forward} yields 
$$
\frac{p_{X_t}(x_t)}{p_{X_{t-1}}(x_{t-1})}
=
\frac{p_{X_{t-1}}(x_{t-1})|\mathrm{det} \nabla \mathcal T_t^{-1}(x_{t})|}{p_{X_{t-1}}(x_{t-1})}
=
|\nabla \mathcal T_t^{-1}(x_{t})|.
$$
Further, for any measurable rectangle $A\times B$ it holds
\begin{align*}
P_{X_{t-1},X_t}(A\times B)
&=\int_A\delta_{\mathcal T_t(x_{t-1})}(B) \dx P_{X_{t-1}}(x_{t-1})
=
P_{X_{t-1}}(A\cap \mathcal T_t^{-1}(B))\\
&=P_{\mathcal T_t(X_{t-1})}(\mathcal T_t(A)\cap B)
=
\int_B\delta_{\mathcal T_t^{-1}(x_t)}(A) \dx P_{X_t}(x_t)\\
&=
(\mathcal R_t \times P_{X_t})(A\times B).
\end{align*}
Since the measurable rectangles are a $\cap$-stable generator of 
$\mathcal B(\R^d)\otimes\mathcal B(\R^d)$, we obtain that $P_{X_{t-1},X_t}=\mathcal R_t \times P_{X_t}$.
By definition, this yields that 
$\mathcal R_t=P_{X_{t-1}|X_t}$ 
such that 
$f_t(\cdot,x_{t})=\frac{\dx \mathcal R_t(x_t,\cdot)}{\dx P_{X_{t-1}|X_t=x_t}}$ 
is given by $f(x_{t-1},x_t)=1$.
\\
ii) Denote by $P_t$ the measure with density $p_t$. 
Since $\mathcal R_t=\mathcal K_t$ 
and using Lemma~\ref{prop_KL_MC_paths}, 
we obtain that the Radon-Nikodym derivative 
$\frac{\dx (\mathcal R_t\times P_{X_t})}{\dx (\mathcal K_t\times P_t)}$ 
is given by 
$f(x_{t-1},x_t)=\tfrac{p_{X_t}(x_t)}{p_t(x_t)}$. 
Further, we have by the detailed balance condition that
$$
\frac{\dx (P_{X_{t-1}|X_t}\times P_{X_t})}{\dx (M_t\times P_t)}
=
\frac{\dx P_{X_{t-1},X_t}}{\dx (P_t\times M_t)}
=
\frac{\dx (P_{X_{t-1}}\times \mathcal K_t)}{\dx (P_t\times \mathcal K_t)}
$$
is given by 
$g(x_{t-1},x_t)=\tfrac{p_{X_{t-1}}(x_{t-1})}{p_t(x_{t-1})}$. 
Thus, we obtain that $\frac{\dx (\mathcal R_t\times P_{X_t})}{\dx (P_{X_{t-1}|X_t}\times P_{X_t})}$ 
is given by $\tfrac{f(x_{t-1},x_t)}{g(x_{t-1},x_t)}=\tfrac{p_{X_t}(x_{t})p_t(x_{t-1})}{p_{X_{t-1}}(x_{t-1})p_t(x_t)}$.
On the other hand, we see by Lemma~\ref{prop_KL_MC_paths} that 
$\frac{\dx (\mathcal R_t\times P_{X_t})}{\dx (P_{X_{t-1}|X_t}\times P_{X_t})}$ is given by $f_t$. 
Therefore, we conclude $f_t=f/g$, i.e.\ 
$$
f_t(x_{t-1},x_t)=\frac{p_{X_t}(x_{t})p_t(x_{t-1})}{p_{X_{t-1}}(x_{t-1})p_t(x_t)}.
$$
Reformulating this equation, proves the claim.
\\[1ex]
iii) Bayes' theorem yields that $P_{X_{t-1}|X_t=x_t}$ has the density 
$$
p_{X_{t-1}|X_t=x_t}(x_{t-1})=p_{X_t|X_{t-1}=x_{t-1}}(x_t)\frac{p_{X_{t-1}}(x_{t-1})}{p_{X_t}(x_t)}.
$$
Since both
$P_{Y_{t-1}|Y_t=x_t}=\mathcal K_t(x_t,\cdot)=P_{X_t|X_{t-1}=x_t}$ 
and 
$P_{X_{t-1}|X_t=x_t}$ 
have a density with respect to the Lebesgue measure, 
we obtain that the Radon-Nikodym derivative 
$f_t(\cdot,x_t)=\frac{\dx P_{Y_{t-1}|Y_t=x_t}}{\dx P_{X_{t-1}|X_t=x_t}}$ 
reads as
$$
f_t(x_{t-1},x_t)=\frac{p_{Y_{t-1}|Y_t=x_t}(x_{t-1})}{p_{X_{t-1}|X_t=x_t}(x_{t-1})}
=
\frac{p_{X_{t}|X_{t-1}=x_t}(x_{t-1})}{p_{X_{t}|X_{t-1}=x_{t-1}}(x_{t})}\frac{p_{X_t}(x_t)}{p_{X_{t-1}}(x_{t-1})}.
$$
Reformulating this equation yields \eqref{eq_logdet_density}.
Now we have for the Langevin transition kernel that 
$$
p_{X_t|X_{t-1}=x}(y)=\mathcal N(y|x-a_1 \nabla u_t(x),a_2^2 I).
$$
Inserting $x_{t-1}$ and $x_t$, using the properties of the normal distribution, we obtain 
$$
p_{X_t|X_{t-1}=x_{t-1}}(x_t)=  a_2^d \, \mathcal N(\eta_t|0,I),
\quad
 p_{X_t|X_{t-1}=x_{t}}(x_{t-1})=  a_2^d \, \mathcal N(\tilde\eta_t|0,I).
$$
Finally, using \eqref{eq_logdet_density}, we get 
$$
\frac{f(x_{t-1},x_t)p_{X_{t-1}}(x_{t-1})}{p_{X_{t}}(x_{t})}
=
\frac{\mathcal N(\tilde \eta_t|0,I)}{\mathcal N(\eta_t|0,I)}
=
\exp\Big( \frac12(\|\eta_t\|^2-\|\tilde\eta_t\|^2) \Big)
$$
which finishes the proof. 
\end{proof} 

The following theorem shows that a SNF 
$((X_0,...,X_T),(Y_T,...,Y_0))$ with $Y_T=X$ 
is "perfectly trained" if and only if 
$P_{X_T}=P_X$ and the distributions $P_{X_{t-1}}$ and $P_{X_t}$ 
before and after any MCMC layer $\mathcal K_t$ are given by the stationary distribution ("equilibrium") of $\mathcal K_t$.
In other words, in the case of a "perfectly trained" SNF, the distribution of $X_{t-1}$ 
before each MCMC layer $\mathcal K_t$ is the same as the distribution of $X_t$ afterward. 
Consequently, we can skip the MCMC layers in this case and still end up with the distribution $P_{X_T}=P_X$.

\begin{theorem}\label{Opt_cond}
Let $((X_0,...,X_T),(Y_T,...,Y_0))$ be a SNF and $Y_T = X$.
Then, the following holds true:
\begin{itemize}
\item[i)] For $t=1,...,T$, we have
$\E_{(x_{t-1},x_t)\sim P_{(Y_{t-1},Y_t)}}\left[\log(f_t(x_{t-1},x_t)) \right]\geq 0$.
\item[ii)] 
If $\mathcal K_t$ fulfills the detailed balance condition \eqref{eq_detailed_balance} with respect to  $P_{t}$ 
with density $p_t$ and $\mathcal R_t=\mathcal K_t$, then
$$
\E_{(x_{t-1},x_t)\sim P_{(Y_{t-1},Y_t)}}\left[ \log \left(f_t(x_{t-1},x_t) \right) \right]=0
$$
if and only if $P_{X_{t-1}}=P_{X_t}=P_t$.
\item[iii)] 
If each layer of the Markov chain is either deterministic with \eqref{kern_det}
or fulfills the detailed balance condition with $\mathcal R_t = \mathcal K_t$,
then
$$
\mathrm{KL}(P_{(Y_0,...,Y_T)},P_{(X_0,...,X_T)})=0
$$ 
if and only if $P_{X_T}=P_{X}$ 
and for any layer $\mathcal K_t$, which fulfills the detailed balance condition, 
we have $P_{X_{t-1}}=P_{X_t}=P_t$.
\end{itemize}
\end{theorem}

\begin{proof}
i)
Since 
\begin{align*}
&\quad\E_{(x_{t-1},x_t)\sim P_{(Y_{t-1},Y_t})} \left[\log \left( f_t(x_{t-1},x_t) \right) \right]
=
\E_{x_t\sim P_{Y_t}}\left[\E_{x_{t-1}\sim P_{Y_{t-1}|Y_t=x_t}} \left(\log(f_t(x_{t-1},x_t)) \right) \right]\\
&=
\E_{x_t\sim P_{Y_t}}\big[\mathrm{KL}(P_{Y_{t-1}|Y_t=x_t},P_{X_{t-1}|X_t=x_t})\big]
\end{align*}
and the integrand is non-negative for all $x_t$, the whole expression is non-negative.
\\[1ex]
ii) By Part i), we conclude that
$
\E_{x_t\sim P_{Y_t}} \big[
\mathrm{KL}(P_{Y_{t-1}|Y_t=x_t},P_{X_{t-1}|X_t=x_t}) \big]=0
$,
if and only if 
$P_{X_{t-1}|X_t=x_t}=P_{Y_{t-1}|Y_t=x_t} = \mathcal R_t(x_t,\cdot) = \mathcal K_t(x_t,\cdot)$
for $P_{Y_t}$-almost every $x_t$. 
In particular, we have 
$
P_{X_{t-1}}\times \mathcal  K_t (A \times B) = P_{(X_{t-1},X_t)} = P_{X_{t-1}|X_t = x_t} \times P_{X_t}
= \mathcal K_t\times P_{X_t}
$.
Due to the detailed balance condition we also have
$P_t\times \mathcal K_t = \mathcal K_t\times P_t$, 
the Radon-Nikodym derivatives 
$$
\frac{\dx (P_{X_{t-1}}\times \mathcal K_t)}{\dx (P_t\times \mathcal K_t)} =
\frac{\dx (\mathcal K_t \times P_{X_t})}{\dx (\mathcal K_t \times P_t)}
$$
coincide. By Lemma~\ref{prop_KL_MC_paths}
these derivatives  are given by
$$
f(x_{t-1},x_t) = \frac{p_{X_{t-1}}(x_{t-1})}{p_t(x_{t-1})}
\quad\text{and}\quad 
g(x_{t-1},x_t) = \frac{p_{X_t}(x_t)}{p_t(x_t)},
$$
respectively.
Since $f$ is constant in $x_t$ and $g$ is constant in $x_{t-1}$, 
we conclude that there exists some constant $c\geq0$ such that 
$p_{X_{t-1}}(x)=c\,p_t(x)$ and $p_{X_t}(x)=c\,p_t(x)$.
Due to the fact that $p_{X_{t-1}}$, $p_{X_t}$ and $p_t$ 
are probability density functions, we get that $c=1$ and
$p_{X_{t-1}}=p_{X_t}=p_t$ and we are done.
\\[1ex]
iii) By Theorem~\ref{thm_KL_paths}, the KL divergence can be decomposed into
\begin{align*}
&\quad\mathrm{KL}(P_{(Y_0,...,Y_T)},P_{(X_0,...,X_T}))
\\
&=\E_{x_T\sim P_X}\Big[\log\Big(\frac{p_X(x_T)}{p_{X_T}(x_T)}\Big)\Big]
+\sum_{t=1}^T
\E_{(x_{t-1},x_t)\sim P_{(Y_{t-1},Y_t})} \left[\log(f_t(x_{t-1},x_t))\right]
\\
&=\mathrm{KL}(P_X,P_{X_T})
+
\sum_{t=1}^T \E_{(x_{t-1},x_t)\sim P_{(Y_{t-1},Y_t})}(\log(f_t(x_{t-1},x_t))).
\end{align*}
Using the non-negativity of the KL divergence and Part i), 
we obtain that every summand is non-negative.
Thus, we have that the above term is equal to zero if and only if every summand is equal to zero.
Now $\mathrm{KL}(P_X,P_{X_T})=0$ if and only if $P_X=P_{X_T}$.
If the $t$-th layer is deterministic, 
then Lemma~\ref{lem_logdets} implies that $f_t=1$ such that
$$
\E_{(x_{t-1},x_t)\sim P_{(Y_{t-1},Y_t})} \left[\log \left(f_t(x_{t-1},x_t) \right) \right] =0.
$$
If the $t$-th layer fulfills the detailed balance condition, then we know by Part ii)  that
$$
\E_{(x_{t-1},x_t)\sim P_{(Y_{t-1},Y_t})} \left[\log \left(f_t(x_{t-1},x_t) \right)\right]=0
$$
if and only if $P_{X_{t-1}}=P_{X_t}=P_t$.
Combining the above arguments yields the claim.
\end{proof}

\subsection{Differentiation of the Loss Function} \label{sec:differentiation}
In order to learn the parameters of the deterministic layers 
in the SNF we will apply a stochastic gradient descent algorithm.
Therefore we briefly discuss the differentiation of the loss function. 
For this purpose, we briefly discuss the derivation of an unbiased estimator for $\nabla_\theta \mathcal L_\mathrm{SNF}(\theta)$
based on samples $(x_0,...,x_T)$ from the path $(Y_0,...,Y_T)$.

Let $((X_0,...,X_T),(Y_T,...,Y_0)) = ((X_0(\omega,\theta),...,X_T(\omega,\theta)),(Y_T(\omega,\theta),...,Y_0(\omega,\theta)))\colon$
\\
$\Omega\times \Theta\to(\R^d)^{2(T+1)}$ 
be a stochastic normalizing flow depending
on some parameters $\theta$ where $Y_T(\omega,\theta) = X(\omega)$ does not depend on $\theta$.
Further, let $h\colon (\R^d)^{T+1}\to\R$ be some differentiable function.
In our setting, 
$$h\coloneqq \log\Big(\frac{\dx P_{(Y_0,...,Y_T)}}{\dx P_{(X_0,...,X_T)}}\Big).$$
In the following, we aim to minimize the loss function
$$
\mathcal L_{\mathrm{SNF}}(\theta)=\E_{(x_0,...,x_T)\sim P_{(Y_0(\cdot,\theta),...,Y_T(\cdot,\theta))}}[h(x_0,...,x_T)]
$$
by a stochastic gradient descent algorithm.
For this purpose, we need to approximate
\begin{equation}\label{eq_gradient}
\nabla_\theta \mathcal L_{\mathrm{SNF}}(\theta)=\nabla_\theta \E_{(x_0,...,x_T)\sim P_{(Y_0(\cdot,\theta),...,Y_T(\cdot,\theta))}}[h(x_0,...,x_T)]
\end{equation}
by a stochastic gradient.
The stochastic gradient of $\mathcal L_\mathrm{SNF}$ is given by the Monte-Carlo approximation of the integral in \eqref{eq_gradient}, i.e.,
\begin{align}
\nabla_\theta \E_{(x_0,...,x_T)\sim P_{(Y_0(\cdot,\theta),...,Y_T(\cdot,\theta))}}[h(x_0,...,x_T)]&\approx \sum_{i=1}^N \nabla_\theta h(Y_0(\omega_i,\theta)...,Y_T(\omega_i,\theta))
\end{align}
where $\omega_1,...,\omega_N$ are i.i.d.\ samples from $\Omega$. 
Under the assumptions that the parameter space is compact and the gradient $\nabla_{\theta} h(Y_0(\omega,\theta)...,Y_T(\omega,\theta))$ exists and is continuous in $\theta$ for almost every $\omega\in\Omega$, the right side of the above formula is an unbiased estimator of $\nabla_\theta\mathcal L_{\mathrm{SNF}}(\theta)$, as
$$
\int_\Omega\nabla_\theta h(Y_0(\omega,\theta)...,Y_T(\omega,\theta))\dx P(\omega)=\nabla_\theta\int_\Omega h(Y_0(\omega,\theta)...,Y_T(\omega,\theta))\dx P(\omega)=\nabla_\theta\mathcal L_\mathrm{SNF}(\theta),
$$
where we used Leibniz integral rule to interchange derivative and integral. However, the continuity assumption is violated in the case of MCMC kernels, such that an unbiased estimate requires some closer considerations, see e.g. \cite{doucetthin} for some work in this direction.
Now, we want to compute the stochastic gradient.
Using the chain rule, it suffices to compute for an arbitrary fixed $\omega\in\Omega$ the derivatives $\nabla_\theta Y_t(\omega,\theta)$ 
for $t=T,...,0$.

Since $Y_T(\omega,\theta)=X(\omega)$ does not depend on $\theta$, 
we have for $t=T$ that $\nabla_\theta Y_T(\omega,\theta)=0$.
For $t=T,...,1$, we distinguish between the three different kinds of layers for computing $\nabla_\theta Y_{t-1}(\omega,\theta)$:
\vspace{0.4cm}

\noindent
\textbf{Deterministic layer:}
Since
$$
Y_{t-1}(\omega,\theta)=\mathcal T_t^{-1}(Y_{t}(\omega,\theta),\theta),
$$
we obtain by the chain rule
$$
\nabla_\theta Y_{t-1}(\omega,\theta)
=
(\nabla_2 \mathcal T_t^{-1})(Y_{t}(\omega,\theta),\theta)+\big[(\nabla_1 \mathcal T_t^{-1})(Y_{t}(\omega,\theta),\theta)\big]\nabla_\theta Y_{t}(\omega,\theta),
$$
where $\mathcal T_t^{-1}(\cdot,\theta)$ is the inverse of $\mathcal T_t(\cdot,\theta)$ and $\nabla_1\mathcal T_t^{-1}$ and $\nabla_2\mathcal T_t^{-1}$ 
are the derivatives of $\mathcal T_t^{-1}$ with respect to the first and second argument.
This formula coincides with the backpropagation of neural networks.
\vspace{0.4cm}

\noindent
\textbf{Langevin layer:}
In this case, we have 
$$
Y_{t-1}(\omega,\theta)=Y_{t}(\omega,\theta)-a_1\nabla u_t (Y_{t}(\omega,\theta))+a_2\xi_t(\omega)
$$
for some standard normally distributed random variable $\xi_t$ which is independent from $Y_T,...,Y_t$. Then, the chain rule implies
$$
\nabla_\theta Y_{t-1}(\omega,\theta)=\nabla_\theta Y_{t}(\omega,\theta)-a_1 \nabla^2 u_t(Y_{t}(\omega,\theta)) 
\nabla_\theta Y_{t}(\omega,\theta).
$$
\textbf{MCMC-layer:}  In the case that we use the kernel $Q_t$ from Remark~\ref{rem_MCMC_step_kernels} (i), we have
$$
Y_{t-1}(\omega,\theta)=Y_{t}(\omega,\theta)+1_{[U_t(\omega),1]}\big(\alpha_t(Y_{t}(\omega,\theta),Y_{t}(\omega,\theta)+\xi_t(\omega))\big)\xi_t(\omega).
$$
Since $U_t\sim\mathcal U_{[0,1]}$, we get almost surely that $U_t(\omega)\neq\alpha_t(Y_{t}(\omega,\theta),Y_{t}(\omega,\theta)+\xi_t(\omega))$. 
Further, we know that $Y_{t}(\omega,\theta)$ is almost surely continuous in $\theta$.
Hence we obtain that 
\\
$1_{[U_t(\omega),1]}\big(\alpha_t(Y_{t}(\omega,\theta),Y_{t}(\omega,\theta)+\xi_t(\omega))\big)$ is locally constant in $\theta$ almost everywhere.
This yields that almost surely
$$
\nabla_\theta [1_{[U_t(\omega),1]}\big(\alpha_t(Y_{t}(\omega,\theta),Y_{t}(\omega,\theta)+\xi_t(\omega))\big)]=0 .
$$
Since also $\nabla_\theta \xi_t(\omega)=0$, we conclude that
$$
\nabla_\theta Y_{t-1}(\omega,\theta)=\nabla_\theta Y_{t}(\omega,\theta)\quad \text{almost surely}.
$$
Similarly, for $Q_t$ as in Remark~\ref{rem_MCMC_step_kernels} (ii), we have that
$$
Y_{t-1}(\omega,\theta)=Y_{t}(\omega,\theta)+1_{[U(\omega),1]} \left( \alpha_t( Y_{t}(\omega,\theta),Y_{t-1}'(\omega,\theta)) \right) \, (a_2\xi_t(\omega)-a_1\nabla u_t(Y_{t}(\omega,\theta))),
$$
where $Y_{t-1}'$ fulfills $P_{Y_{t},Y_{t-1}'}=P_{Y_t}\times Q_t$.
Then, the application of the chain and product rule yields that $\nabla_\theta Y_{t-1}(\omega,\theta)$ is given by
\begin{align*}
\nabla_\theta Y_{t}(\omega,\theta)-1_{[U(\omega),1]} \left( \alpha_t( Y_{t}(\omega,\theta),Y_{t-1}'(\omega,\theta)) \right) a_1\nabla^2 u_t(Y_{t}(\omega,\theta)) 
\nabla_\theta Y_{t}(\omega,\theta).
\end{align*}
almost surely.

\section{Conditional Stochastic Normalizing Flows for Inverse Problems} \label{sec:inverse}
So far we have only considered the task of sampling from $P_X$ using a (simpler) distribution $P_Z$.
In inverse problems, we have a more general setting.
Let $X\colon \Omega \rightarrow \R^d$ be random variable with prior distribution $P_X$ and 
let $Y\colon \Omega \rightarrow \R^{\tilde d}$ be defined by
\begin{equation} \label{inv_model}
Y = F(X) + \eta,\qquad X\sim P_X
\end{equation}
for some (ill-posed), not necessary linear operator $F\colon \R^d \rightarrow \R^{\tilde d}$ and a random variable 
$\eta \colon \Omega \rightarrow \R^{\tilde d}$ for the noise. 
Now we aim to sample from the posterior distribution 
$P_{X|Y=y}$ by taking as input the observation $y\in\R^{\tilde d}$ and samples from a random variable $Z$. 
That is, we aim to train one flow model, which is able to sample from all posterior distributions $P_{X|Y=y}$ with $y\in\R^{\tilde d}$, where $y$
is taken as an input.
For this purpose, 
we combine the ideas of conditional INNs
and SNFs.
We would like to remark that the posterior distribution $P_{X|Y=y}$ heavily depends on the prior distribution $P_X$
in the sense that replacing the prior distribution $P_X$ can result into a completely different posterior $P_{X|Y=y}$ even
if the operator $F$ and the noise model $\eta$ remain to be the same as before.

A \textit{conditional SNF} conditioned to $Y$
as a pair of sequences $\left((X_0,...,X_T),(Y_T,...,Y_0) \right)$ of random variables $X_t,Y_t\colon\Omega\to\R^d$ such that
\begin{itemize}
\item[cP1)] the conditional distributions $P_{X_t|Y=y}$ and $P_{Y_t|Y=y}$ have densities 
$$
p_{X_t}(y,\cdot)\colon\R^{d_t}\to\R_{>0},\quad\text{and}\quad p_{Y_t}(y,\cdot)\colon\R^{d_t}\to\R_{>0}
$$
for $P_Y$-almost every $y$ and
all $t=1,\ldots,T$,
\item[cP2)] for $P_Y$-almost every $y$, 
there exist Markov kernels $\mathcal K_t\colon\R^{\tilde d}\times\R^{d}\times \mathcal B(\R^{d})\to[0,1]$ 
and $\mathcal R_t\colon\R^{\tilde d}\times \R^{d}\times \mathcal B(\R^{d})\to[0,1]$
such that
\begin{align*}
P_{(X_0,...,X_T)|Y=y}=P_{X_0} \times \mathcal K_1(y,\cdot,\cdot)\times\cdots\times \mathcal K_T(y,\cdot,\cdot),\\
P_{(Y_T,...,Y_0)|Y=y}=P_{Y_T} \times \mathcal R_T(y,\cdot,\cdot)\times\cdots\times \mathcal R_1(y,\cdot,\cdot).
\end{align*}
\item[cP3)] for $P_{Y,X_t}$-almost every pair $(y,x)\in\R^{\tilde d}\times\R^{d}$,
the measures $P_{Y_{t-1}|Y_t=x,Y=y}$ and \\ 
$P_{X_{t-1}|X_t=x,Y=y}$
are absolute continuous with respect to each other.
\end{itemize}
For applications, one usually sets
$
X_0=Z,
$
where $Z$ is a random variable, which is easy to sample from 
and independent from $Y$, i.e., for every $y\in\R^{\tilde d}$ we initialize $P_{X_0|Y=y}=P_Z$. Then, we aim to approximate 
for every $y\in\R^{\tilde d}$ the distribution $P_{X|Y=y}$ by $P_{X_T|Y=y}$. On the other hand, $Y_T$ is usually defined by
$
P_{Y_T|Y=y}=P_{X|Y=y}
$
and $P_{Y_0}$ should approximate the latent distribution $P_Z$.

\begin{remark}\label{rem_conditional_snf}
Let $((X_0,...,X_T),(Y_T,...,Y_0))$ be a conditional SNF.
Then, by definition, the pair $((X_0^y,...,X_T^y),(Y_T^y,...,Y_0^y))$ with 
$$
P_{(X_0^y,...,X_L^y)}=P_{(X_0,...,X_L)|Y=y}\quad
\text{and}\quad
P_{(Y_T^y,...,Y_0^y)}=P_{(Y_T,...,Y_0)|Y=y}
$$
is a SNF as in Section~\ref{sec:SNF} for $P_Y$-almost every observation $y\in\R^{\tilde d}$.
From this viewpoint, a conditional SNF can be viewed as a \emph{family of SNFs}, where each element approximates the posterior distribution
$P_{X|Y=y}$ by $P_{X_T^y}=P_{X_T|Y=y}$ for one $y\in\R^{\tilde d}$. \hfill $\square$
\end{remark}

We learn the parameters in the deterministic layers of a conditional SNF by minimizing the loss function
\begin{equation} \label{loss_cond}
\mathcal L_\mathrm{cSNF}(\theta)= \mathrm{KL}(P_{(Y,Y_0,...,Y_T)},P_{(Y,X_0,...,X_T)}).
\end{equation}
The following corollary of Theorem~\ref{thm_KL_paths} 
establishes the incorporated KL-divergences.

\begin{corollary}\label{cor:KL_conditional}
Let $((X_0,\ldots,X_T),(Y_T,\ldots,Y_0))$ be a conditional SNF conditioned to $Y$. 
and let $P_{X_0|Y=y}=P_{Z}$ and $P_{Y_T|Y=y}=P_{X|Y=y}$.
Then the loss function $\mathcal L_\mathrm{cSNF}$ in \eqref{loss_cond} is given by

\begin{align*}
\mathcal L_\mathrm{cSNF}(\theta)&=\mathrm{KL}(P_{(Y,Y_0,...,Y_T)},P_{(Y,X_0,...,X_T)})
=
\E_{y\sim P_Y} \Big[\E_{(x_0,...,x_T)\sim P_{(Y_0^y,...,Y_T^y)}}
\Big[\\
& \qquad -\log(p_Z(x_0))+\sum_{t=1}^T\log \Big(\frac{f_t^y(x_{t-1},x_t) p_{X^y_{t-1}}(x_{t-1})}{p_{X^y_{t}}(x_{t})}\Big)
\Big]\Big] 
+ \mathrm{const},
\end{align*}
where $(X_0^y,...,X_T^y)$ and $(Y_0^y,...,Y_T^y)$ are defined as in Remark~\ref{rem_conditional_snf}
and $f_t^y(\cdot,x_t)$ is the Radon-Nikodym derivative $\frac{\dx P_{Y_{t-1}|Y_t=x_t,Y=y}}{\dx P_{X_{t-1}|X_t=x_t,Y=y}}$.
\end{corollary}

\begin{proof}
Using $P_{(Y,Y_0,...,Y_T)} = P_Y\times P_{(Y_0,...,Y_T)|Y}$ 
and Lemma~\ref{prop_KL_MC_paths}, we obtain
\begin{align*}
&\quad\mathrm{KL}(P_{(Y,Y_0,...,Y_T)},P_{(Y,X_0,...,X_T)})\\
&=\
E_{(y,x_0,...,x_T)\sim P_{(Y,Y_0,...,Y_T)}}
\Big[\log\Big(
\frac{\dx P_{(Y,Y_0,...,Y_T)}}{\dx P_{(Y,X_0,...,X_T)}}(y,x_0,...,x_T)
\Big)\Big]\\
&=
\E_{y\sim P_Y}\Big[\E_{(x_0,...,x_T)\sim P_{(Y_0,...,Y_T)|Y=y}}
\Big[\frac{\dx P_{(Y_0,...,Y_T)|Y=y}}{\dx P_{(X_0,...,X_T)|Y=y}}(x_0,...,x_T)\Big]\Big]
\end{align*}
which is equal to
$
\E_{y\sim P_Y}
\big[\mathrm{KL}(P_{(Y_0^y,...,Y_T^y)},P_{(X_0^y,...,X_T^y)})\big]
$.
Since $((X_0^y,...,X_T^y),(Y_T^y,...,Y_0^y))$ is a stochastic normalizing flow, 
Theorem~\ref{thm_KL_paths} yields 
\begin{align*}
&\mathrm{KL}(P_{(Y,Y_0,...,Y_T)},P_{(Y,X_0,...,X_T)}) 
= 
\E_{y\sim P_Y}\big[\mathrm{KL}(P_{(Y_0^y,...,Y_T^y)},P_{(X_0^y,...,X_T^y)} )\big]\\
& = \E_{y\sim P_Y}\Big[\E_{(x_0,...,x_T)\sim P_{(Y_0,...,Y_T)|Y=y}} \Big[\\
& \quad
-\log(p_Z(x_0))
+
\log(p_{X|Y=y}(x_T))
+
\sum_{t=1}^T \log\Big(\frac{f_t^y(x_{t-1},x_t) p_{X^y_{t-1}}(x_{t-1})}{p_{X^y_{t}}(x_{t})}\Big)
\Big]\Big].
\end{align*}
Finally, the second summand given by 
$$\E_{y\sim P_Y}[\E_{x_T\sim P_{X|Y=y}}[\log(p_{X|Y=y}(x_T))]]$$ is a constant. 
This finishes the proof.
\end{proof}

The adaption of the deterministic and stochastic layers to the conditional setting is 
outlined in Appendix~\ref{sec_cond_layers}.
Further, for the computation of the term
$
f_t^y( x_{t-1},x_t ) p_{X^y_{t-1}}(x_{t-1})/p_{X^y_{t}}(x_{t})
$ corresponding to the three different layers of conditional SNFs, we can adapt Theorem~\ref{lem_logdets} in a straightforward way
for conditional SNFs.

Finally, we let us have a look at the KL divergence in the loss function and discuss the consequences
if we switch the order of the Markov pairs.
  
\begin{remark}
The KL divergence is not symmetric. Therefore we could also train (conditional) SNFs using instead of ${\mathcal L}_\mathrm{cSNF}$
the switched KL loss function
$$
\tilde {\mathcal L}_\mathrm{cSNF}(\theta)=\mathrm{KL}(P_{(Y,X_0,...,X_T)},P_{(Y,Y_0,...,Y_T)})
$$
or a convex combination of both. 
For normalizing flows this was done e.g. in \cite{AFHHSS2021,HINT}.
In the literature, loss functions similar to $\mathcal L_\mathrm{cSNF}$ are sometimes called \emph{forward} KL, 
while loss functions related to $\tilde {\mathcal L}_\mathrm{cSNF}$ are known as \emph{backward} KL.
Using similar computations as in Corollary~\ref{cor:KL_conditional}, we obtain that 
\begin{align*}
&\tilde {\mathcal L}_\mathrm{cSNF}(\theta)=\mathrm{KL}(P_{(Y,X_0,...,X_T)},P_{(Y,Y_0,...,Y_T)})
=
\E_{y\sim P_Y}
\Big[\E_{(x_0,...,x_T)\sim P_{(X_0^y,...,X_T^y)}}
\Big[  \\
& \qquad
-\log(p_{Y|X=x_T}(y))- \log(p_X(x_T))
- \sum_{t=1}^T \log\Big(\frac{f_t^y(x_{t-1},x_t) p_{X^y_{t-1}}(x_{t-1})}{p_{X^y_{t}}(x_{t})} \Big)
\Big]
\Big]+ \mathrm{const},
\end{align*}
where $p_{Y|X=x_T}(y)$ is determined by the noise term in \eqref{inv_model}.
Then the requirements for minimizing $\mathcal L_\mathrm{cSNF}$ and $\tilde{\mathcal L}_\mathrm{cSNF}$ differ:
\\[0.5ex]
$\quad -$
Forward KL $\mathcal L_\mathrm{cSNF}$: we need samples $(x,y)$ from the joint
distribution $P_{X,Y}$.
\\
$\quad -$ Backward KL $\tilde {\mathcal L}_\mathrm{cSNF}$:
we need samples from $Y$ as well as knowledge over the prior distribution, the forward operator $F$ and the noise distribution of 
$\eta$ for evaluating $p_X$ and $p_{Y|X=x}$.
\\[0.5ex]
Further, the loss functions $\mathcal L_\mathrm{cSNF}$ and $\tilde{\mathcal L}_\mathrm{cSNF}$ have different approximation
properties, see e.g.~\cite{minka2005divergence}. 
By definition, the KL divergence $\mathrm{KL}(P_1,P_2)$ 
between two probabiltiy measures $P_1$ and $P_2$ is large if there exist regions $A\in\mathcal B(\R^d)$ with $P_2(A)\ll P_1(A)$.
Consequently, the loss function $\mathcal L_\mathrm{cSNF}$ penalizes samples from the data distribution $P_{X|Y=y}=P_{Y_T|Y=y}$ which are
out of the distribution $P_{X_T|Y=y}$ generated by the conditional SNF.
Therefore, the forward KL $\mathcal L_\mathrm{cSNF}$ is often called \emph{mode-covering} as the reconstruction includes all
samples from the data distribution.
Conversely, the loss function $\tilde{\mathcal L}_\mathrm{cSNF}$ penalizes samples from the conditional SNF $P_{X_T|Y=y}$ which are
not included within the data distribution $P_{X|Y=y}=P_{Y_T|Y=y}$.
Thus, the backward KL $\tilde{\mathcal L}_\mathrm{cSNF}$ is \emph{mode-seeking} in the sense that it enforces that all 
samples from the conditional SNF are likely under the data distribution. \hfill $\square$
\end{remark}

\section{Numerical Results} \label{sec:numerics}
In this section, we demonstrate the performance of our conditional 
SNFs by two examples.
The first one is artificial and 
uses properties of Gaussian mixture models to get a ground truth.
The second example comes from a real-world application in
scatterometry. 
All implementations are done in Python using Pytorch and the FrEIA framework\footnote{available at \url{https://github.com/VLL-HD/FrEIA}}.
The code is availible online\footnote{\url{https://github.com/PaulLyonel/conditionalSNF}}.

\subsection{Posterior Approximation for Gaussian Mixtures }\label{sec:mm}
To verify that our proposed methods yield the correct posteriors, 
we apply our framework to a linear inverse problem, where we can analytically infer the ground truth
by the following lemma. Its simple proof can be found, e.g. in \cite{GFO2017}.
                             
\begin{lemma} \label{mm}
Let $X \sim \sum_{k=1}^K w_k \mathcal N(m_k,\Sigma_k)$.
Suppose that 
$$
Y=AX+\eta,
$$
where
$A: \R^d \rightarrow \R^{\tilde d}$ 
is a linear operator and we have Gaussian noise 
$\eta \sim N(0,b^2 I)$. Then   
$$
P_{X|Y=y} \propto \sum_{k=1}^K \tilde w_k \mathcal N(\cdot|\tilde m_k,\tilde \Sigma_k),
$$
where $\propto$ denotes equality up to a multiplicative constant and
$$
\tilde \Sigma_k \coloneqq (\tfrac{1}{b^2}A^\tT A+\Sigma_k^{-1})^{-1},
\qquad 
\tilde m_k \coloneqq \tilde\Sigma_k (\tfrac1{b^2}A^\tT y+\Sigma_k^{-1}\mu_k)
$$
and
$$
\tilde w_k \coloneqq \frac{w_k}{|\Sigma_k|^{\tfrac12}} \exp\left(\frac12 (\tilde m_k \tilde \Sigma_k^{-1} \tilde m_k - m_k \Sigma_k^{-1} m_k)\right).
$$
\end{lemma}
In the following, we consider 
$A\coloneqq 0.1\, \mathrm{diag}((\tfrac{1}{n})_{n=1}^d)$, 
where $d=\tilde d = 100$ and  $\eta\sim\mathcal N(0,0.1\, I)$.
As prior distribution $P_X$ we choose a Gaussian mixture model with $K=12$ components,
where we draw the means $m_k$ uniformly from $[-1,1]^d$ and set $\Sigma_k\coloneqq 0.01^2\,I$.

\paragraph{Model parameters and training}
We will approximate the posterior distribution $P_{X|Y=y}$ for arbitrary observations $y$ using a conditional SNF with $T=4$ layers, where the layers themselves are defined as follows:
\\

- $t=1,3$: $\mathcal K_t(y,x,A)=\delta_{\mathcal T_t(y,x)}(A)$ deterministic layer, where
$\mathcal T_t$ is a conditional INN with $L=4$ layers, where each subnetwork has two hidden layers with $128$ neurons.
\\

- $t = 2,4$: $\mathcal K_t$ involves $3$ MCMC steps where the Markov kernel $Q_t$ is given by \eqref{eq_MCMC_step_MALA}, using the step sizes of $a_1 = 10^{-4}$ and $ a_2 = \sqrt{2a_1}$.
\\
We compare the results of the conditional SNF with a conditional INN  with $L=8$ layers, where each subnetwork has two hidden layers with $128$ neurons.
Note that the conditional INN and the conditional SNF have the same number of parameters. We do not use any permutations.

We train both networks with the Adam optimizer \cite{KB2015} with a batch size of $1024$ for $20000$ steps and a learning rate of $10^{-4}$ for the loss function \eqref{KL-loss}.

\paragraph{Quality measure}
To measure the quality of the approximations of $P_{X|Y=y}$
by $P_{X_T|Y=y}$ (conditional SNF) and by $\mathcal T(y,\cdot)_\#P_Z$ (conditional INN),  we
generate $5000$ samples from each distribution and evaluate the Wasserstein-$1$ distance $W_1$
of the arising point measures\footnote{For the computation of the Wasserstein distance we use the Python package POT (Python Optimal Transport) \cite{POT}.}.
We repeat this procedure with $10000$ samples.

\paragraph{Results}
We approximate the posterior $P_{X|Y=y_i}$ by $P_{X_T|Y=y_i}$ and $\mathcal T(y_i,\cdot)_\#P_Z$ for $100$ i.i.d.\ samples $y_i$ from $Y$. 
The averaged (approximated) distances 
$W_1(P_{X|Y=y_i},P_{X_T|Y=y_i})$ and $W_1(P_{X|Y=y_i},\mathcal T(y_i,\cdot)_\#P_Z)$ over 5 training runs 
are given by $1.998\pm 0.09$ and  $2.263 \pm 0.14$ respectively. We observe that the conditional SNF performs better.

To verify, if $5000$ samples are enough to approximate the Wasserstein distance in $d = 100$, we also evaluate with $10000$ samples and obtain that the means over the 5 training runs only differ very slightly (i.e. $1.992$ for the conditional SNF and $2.256$ for the conditional INN).

For two exemplar values of $y_i$ we plotted histograms of (some marginals of) the samples generated by the ground truth, 
the conditional SNF and the conditional INN in Figure~\ref{fig:fig}.
Here, one can see the topological issues of the conditional INN, which has difficulties to distribute mass to very peaky 
modes and therefore moves some mass in between them. 
The MALA layers in the conditional SNF overcome these topological constraints.

\begin{figure}
\centering
\begin{subfigure}{.44\textwidth}
  \centering
  \includegraphics[width=\linewidth]{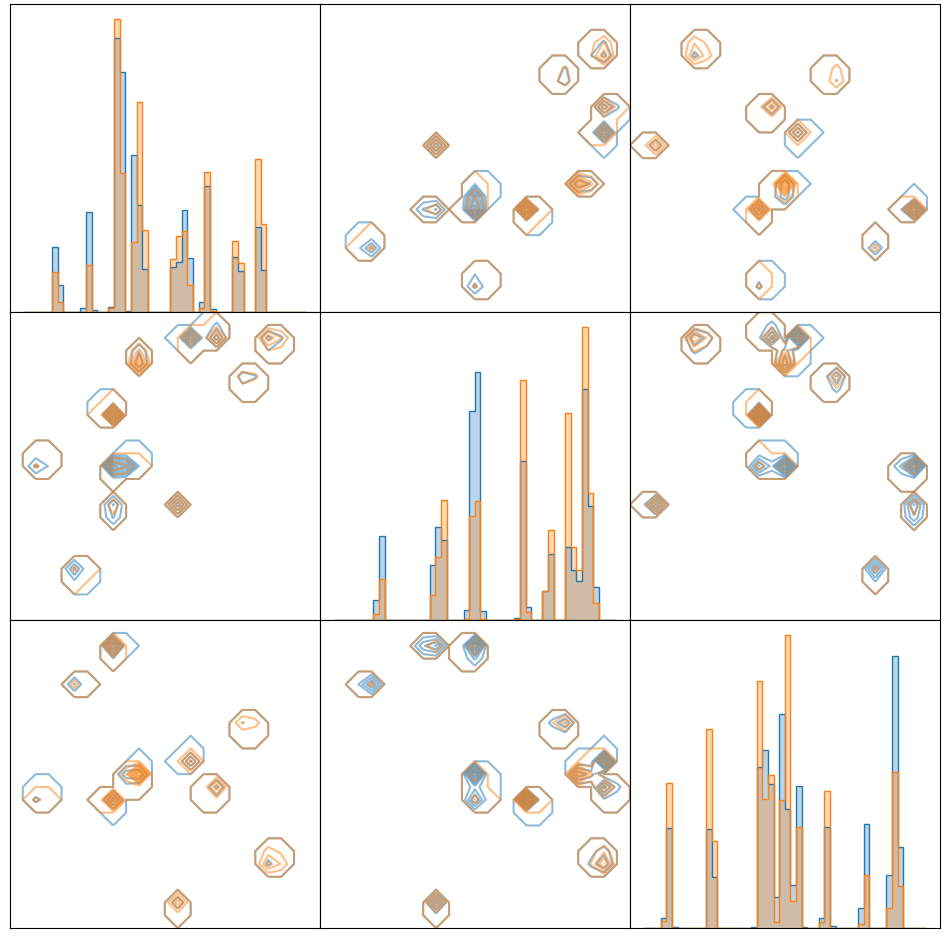}  
\end{subfigure}
\begin{subfigure}{.44\textwidth}
  \centering
  \includegraphics[width=\linewidth]{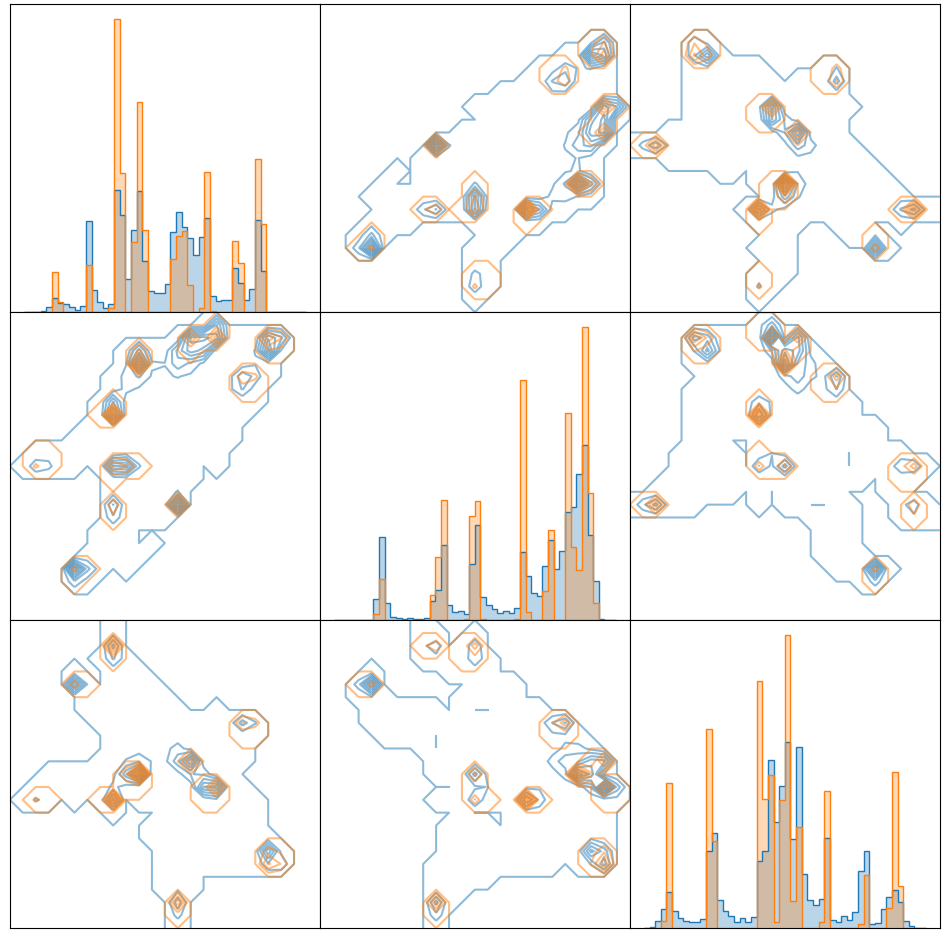}  
\end{subfigure}
\begin{subfigure}{.44\textwidth}
  \centering
  \includegraphics[width=\linewidth]{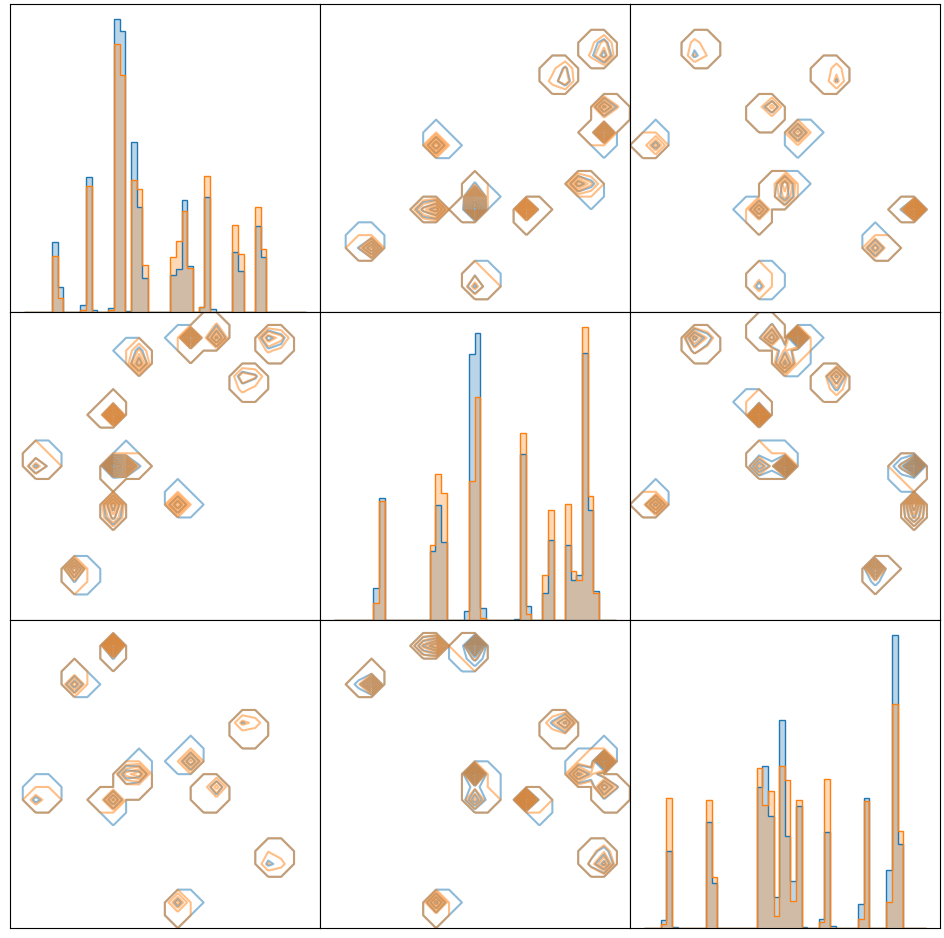}  
\end{subfigure}
\begin{subfigure}{.44\textwidth}
  \centering
  \includegraphics[width=\linewidth]{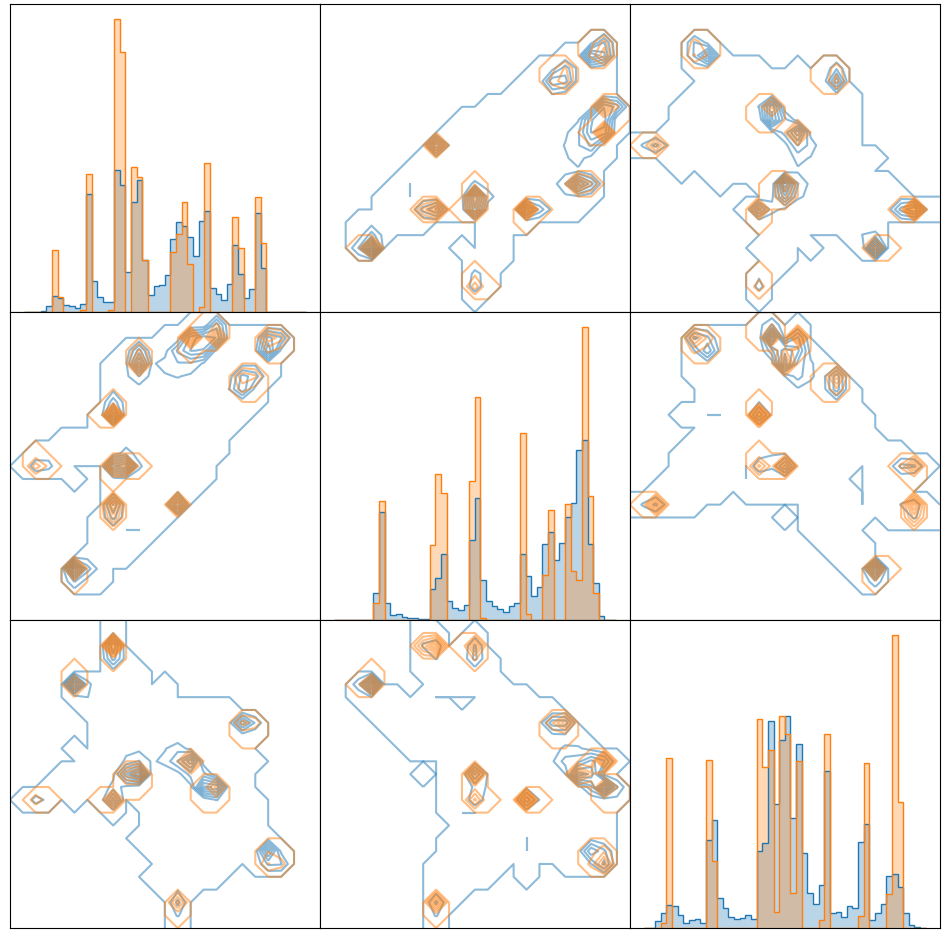}  
\end{subfigure}
\caption{Histograms of the first, $50$-th and $100$-th marginal of the ground truth (orange) and the posterior reconstructions using a conditional SNF (left, blue) and a conditional INN (right, blue) for
two different samples of $Y$.
On the diagonal we plot the one histograms
of the one-dimensional marginals, on the off-diagonal we plot the potentials of the two 
dimensional marginals.}
\label{fig:fig}
\end{figure}

\subsection{Example from Scatterometry}\label{sec:scatterometry}
Next, we apply the conditional SNFs to a real world inverse problem, where the nonlinear forward operator $F\colon \R^3 \rightarrow \R^{23}$ describes the diffraction of monochromatic lights on line gratings and is a non-destructive technique to determine the structures of photo masks \cite{heidenreich2015bayesian,scatter}. 
The parameters in $x$-space describe the geometry of the line gratings and $Y = F(X) + \eta$ the 
diffraction pattern. 
The inverse problem consists now in recovering the geometry  given an observation $y$. 
We assume that the noise $\eta$ is mixed additive and multiplicative Gaussian noise, i.e., $\eta=a F(X) \eta_1+b\eta_2$, where $\eta_1,\eta_2\sim \mathcal{N}(0,I)$ and $a,b>0$ are some constants. 
Then, the conditional distribution $P_{Y|X=x}$ is given by $\mathcal N \left(F(x),(a^2 F(x)^2+b^2) \, I \right)$
Here $b$ represents the strength of background noise, while $a$ controls the strength of fluctuations depending on the forward model. 

In scatterometry, the forward operator $F$ is known from physics, but its evaluation requires the 
solution of some partial differential equation \cite{scatter}, which is slow and computationally costly.
Therefore, we sample $N=10000$ data points $x_i$ uniformly from $[-1,1]^3$ and evaluate the exact forward
operator $F(x_i)$, $i=1,...,N$ (which is computationally costly).
Then, we approximate $F$ by a feed-forward neural network $\tilde F$ with $3$ hidden layers and $256$ 
neurons in each hidden layer by minimizing the loss function
$$
\sum_{i=1}^N \|\tilde F(x_i)-F(x_i)\|^2.
$$
Throughout this section, we will use this approximation $\tilde F$ as our forward operator $F$.

Since we do not have any prior information about the parameters $x$, we choose the prior distribution
$P_X$ as the uniform distribution on $[-1,1]^3$.
For the conditional SNF we assumed that $P_X$ has a strictly positive density $p_X$.
To fulfill this assumption, we relax the probability density function of the uniform distribution
for $x=(x_1,x_2,x_3)\in\R^3$ by
$$
p_X(x)\coloneqq q(x_1)q(x_2)q(x_3), \quad q(x)\coloneqq \begin{cases}\frac{\alpha}{2\alpha+2}\exp(-\alpha(-1-x)),&$for $x<-1,\\
\frac{\alpha}{2\alpha+2},&$for $x\in[-1,1],\\\frac{\alpha}{2\alpha+2}\exp(-\alpha(x-1)),&$for $x>1,\end{cases}
$$
where $\alpha\gg 0$ is some constant. 
Note that for large $\alpha$ and $x_i$ outside of $[-1,1]$ the function $q$ becomes small such that $p_X$ 
is small outside of $[-1,1]^3$.
In our numerical experiments, we choose $\alpha=1000$.
Now the density of the posterior distribution $P_{X|Y=y}$ can be evaluated up to a multiplicative 
constant by Bayes' theorem as
$
p_{X|Y=y}(x)\propto p_{Y|X=x}(y)p_X(x).
$
We can assume for this experiment that $a = 0.2$ and $b = 0.01$.

\paragraph{Model parameters and training}
We train a conditional SNF with $T=8$ layers similarly to the previous example.
The layers $\mathcal K_t$ for $t=1,3,5,7$ are deterministic layers with
conditional INNs with $L=1$ layers, where the subnetworks has two hidden
layers with $64$ neurons in each hidden layer. We do not use any permutations.
The layers $\mathcal K_t$ for $t=2,4,6,8$ consist of $10$ MCMC steps using the kernel 
$Q_t$ as defined in \eqref{eq_MCMC_step_MH}.
Here, we set $\sigma=0.4$.

As a comparison, we also implement a conditional INN  with $L=4$ layers, 
where each subnetwork has two hidden layers with $64$ neurons in each hidden layer.
Note that this conditional INN has exactly the same number of parameters 
as the conditional SNF.
We train both networks using the Adam optimizer with a batch size of $1600$ and a learning rate of $10^{-3}$ 
for the loss function \eqref{KL-loss}.
For the SNF, we run $40$ epochs, which takes approximately $50$ seconds.
Since for the conditional INN it takes longer until the loss saturates, 
we train the conditional INN for $5000$ epochs, which takes approximately $8$ minutes.
Each epoch consists of $8$ steps of the Adam optimizer.

After the training, we approximate the posterior $P_{X|Y=y}$ by the measure $P_{X_T|Y=y}$ for the
conditional SNF and by the measure $\mathcal T(y,\cdot)_\#P_Z$ for the conditional INN.

\paragraph{Quality measure}
Since we do know the ground truth, we use samples generated by the 
Metropolis-Hastings algorithm as a baseline.
To generate a sample from $P_{X|Y=y}$, we run $1000$ steps of the Metropolis Hastings algorithm, i.e., 
apply $1000$ times the kernel $\mathcal K_{\mathrm{MH}}$ from \eqref{eq_MH_kernel}, where the density 
$p$ in \eqref{eq_MH_kernel} is replaced by $p_{X|Y=y}$.

To evaluate the quality of the approximation $P_{X_T|Y=y}$ 
generated by the conditional SNF of $P_{X|Y=y}$, we approximate 
$\mathrm{KL}(P_{X|Y=y},P_{X_T|Y=y})$ as follows: 
Let $\mathcal X=\{x_i:i=1,...,N\}$, $N=160000$ be samples of $P_{X_T|Y=y}$ generated by the  conditional SNF and let
$\tilde{\mathcal{X}}= \{\tilde x_i:i=1,...,N\}$ be samples from $P_{X|Y=y}$ generated by the
Metropolis-Hastings algorithm.
We split our domain $[-1,1]^3$ into $50^3$ cubes $(C_{ijk})_{i,j,k=1}^{50}$ of size $\tfrac{1}{25}$.
Then, we approximate $\mathrm{KL}(P_{X|Y=y},P_{X_T|Y=y})$ by 
$\mathrm{KL}(\mu_{\mathrm{MH}},\mu_{\mathrm{SNF}})$, where $\mu_{\mathrm{MH}}$ and $\mu_{\mathrm{SNF}}$
are the discrete measures
$$
\mu_\mathrm{MH}(i,j,k)\coloneqq\frac{|\tilde{\mathcal X}\cap C_{ijk}|}{N},\quad 
mu_\mathrm{SNF}(i,j,k) \coloneqq \frac{|\mathcal X\cap C_{ijk}|}{N}.
$$
We approximate the the KL divergence $\mathrm{KL}(P_{X|Y=y},\mathcal T_\#P_Z)$ analogously.

\paragraph{Results}
We approximate the posterior $P_{X|Y=y_i}$ by $P_{X_T|Y=y_i}$ and by $\mathcal T(y_i,\cdot)_\#P_Z$ 
for $100$ i.i.d.\ samples $y_i$ of $Y$. Then, the averaged (approximated) Kullback-Leibler divergences 
$\mathrm{KL}(P_{X|Y=y_i},P_{X_T|Y=y_i})$ and 
$\mathrm{KL}(P_{X|Y=y_i},\mathcal T(y_i,\cdot)_\#P_Z)$ over $100$ observations $y$
is given by $0.58\pm 0.20$ and $0.84\pm0.29$, respectively.

To check whether $50^3$ cubes suffice to approximate the Kullback-Leibler divergence, we also took the same models and 
evaluated the average KL over 100 observations $y$ with $75^3$ cubes and $540000$ samples. 
Here, we obtained a mean of $0.54$ for the conditional SNF and $0.78$ for the conditional INN, which is close to the 
results for $50^3$ bins.

For two exemplar values $y_i$, we plotted the histograms of the samples generated by
Metropolis Hastings, the conditional SNF and the conditional INN in Figure~\ref{fig:scattero}.
We observe that the reconstruction using the conditional SNF fits the true posterior generated
by the Metropolis-Hastings algorithm better than the reconstruction using the conditional INN, even though
the training time of the conditional INN was much longer.
In particular, the conditional INN has problems to separate the different modes
in the multimodal second component of the posterior.

\begin{figure}
\centering
\begin{subfigure}[t]{.39\textwidth}
  \centering
  \includegraphics[width=\textwidth]{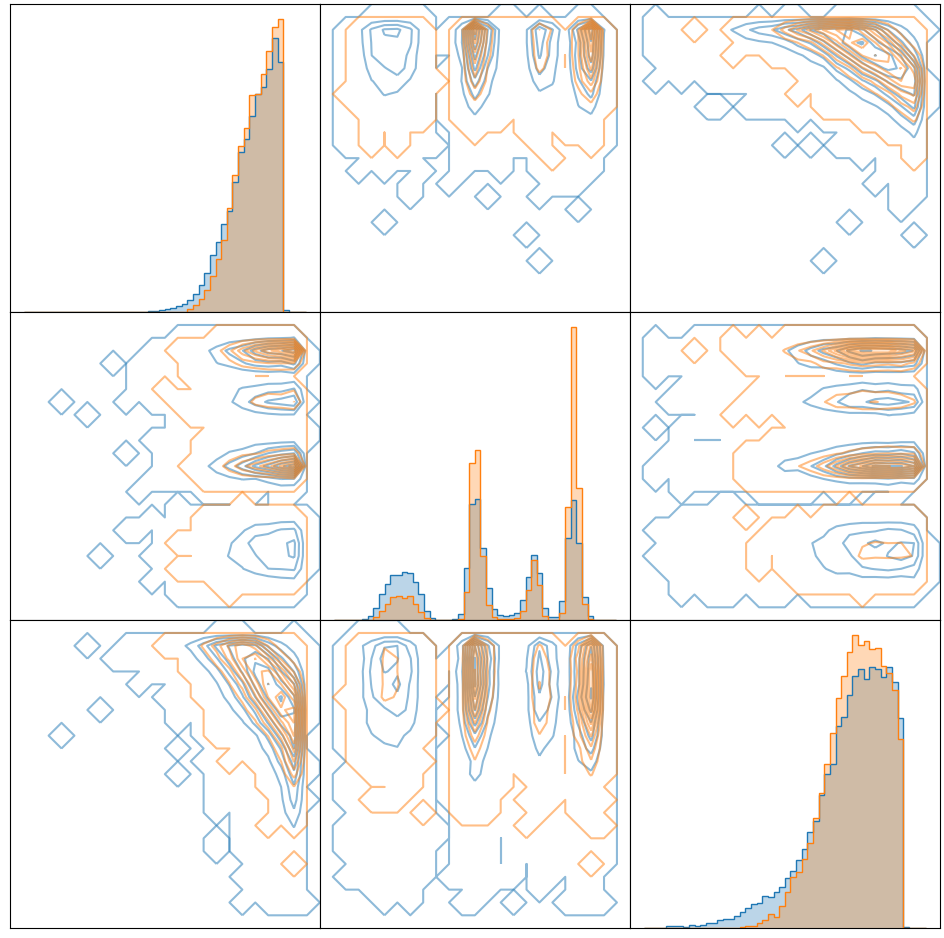}  
\end{subfigure}
\begin{subfigure}[t]{.39\textwidth}
  \centering
  \includegraphics[width=\textwidth]{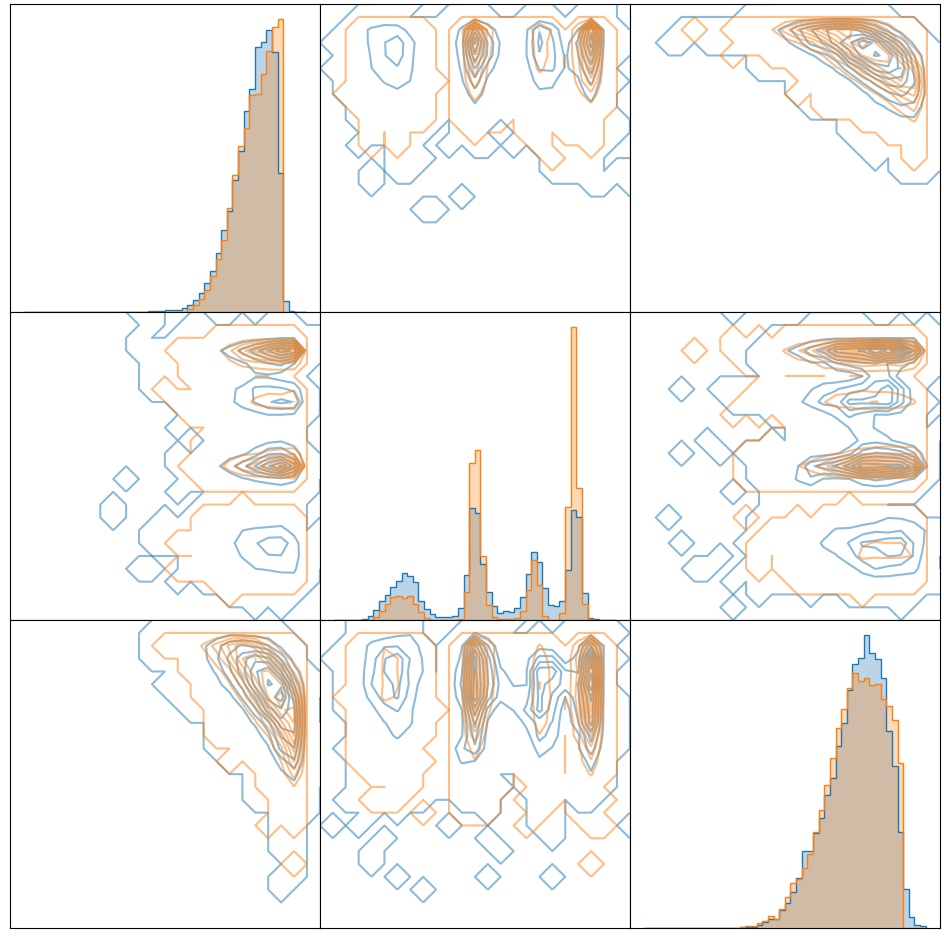}  
\end{subfigure}
\begin{subfigure}[t]{.39\textwidth}
  \centering
  \includegraphics[width=\textwidth]{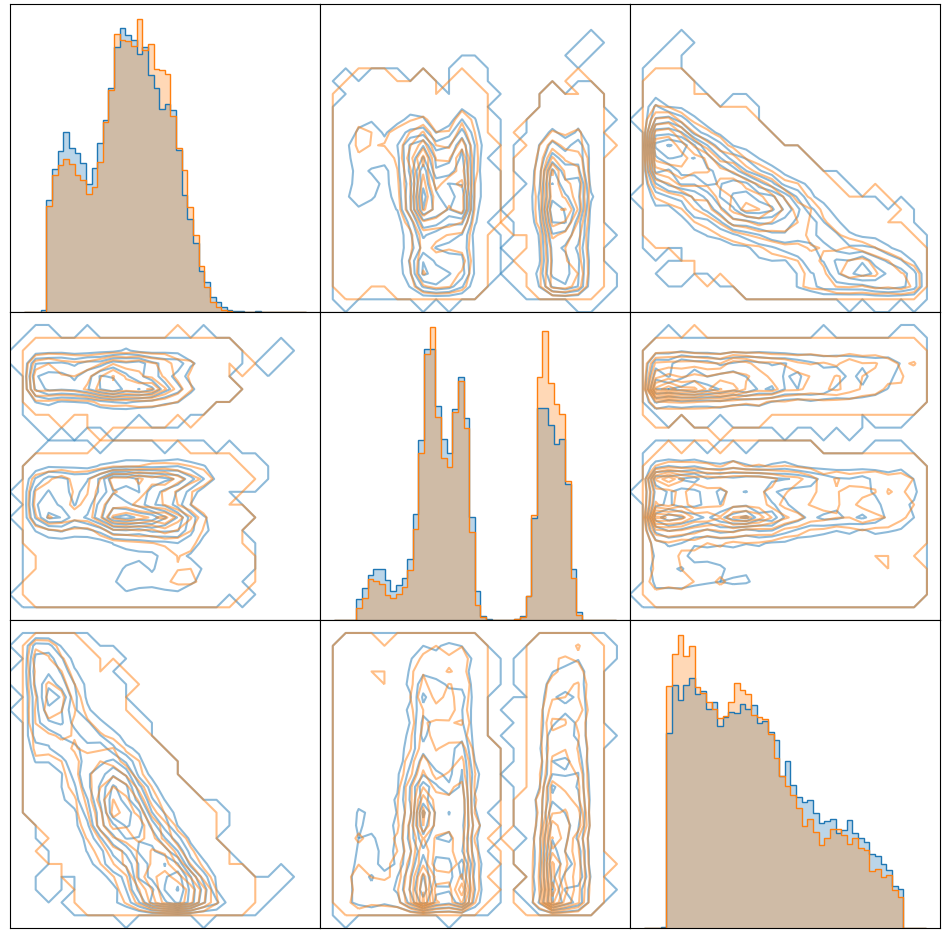}  
\end{subfigure}
\begin{subfigure}[t]{.39\textwidth}
  \centering
  \includegraphics[width=\textwidth]{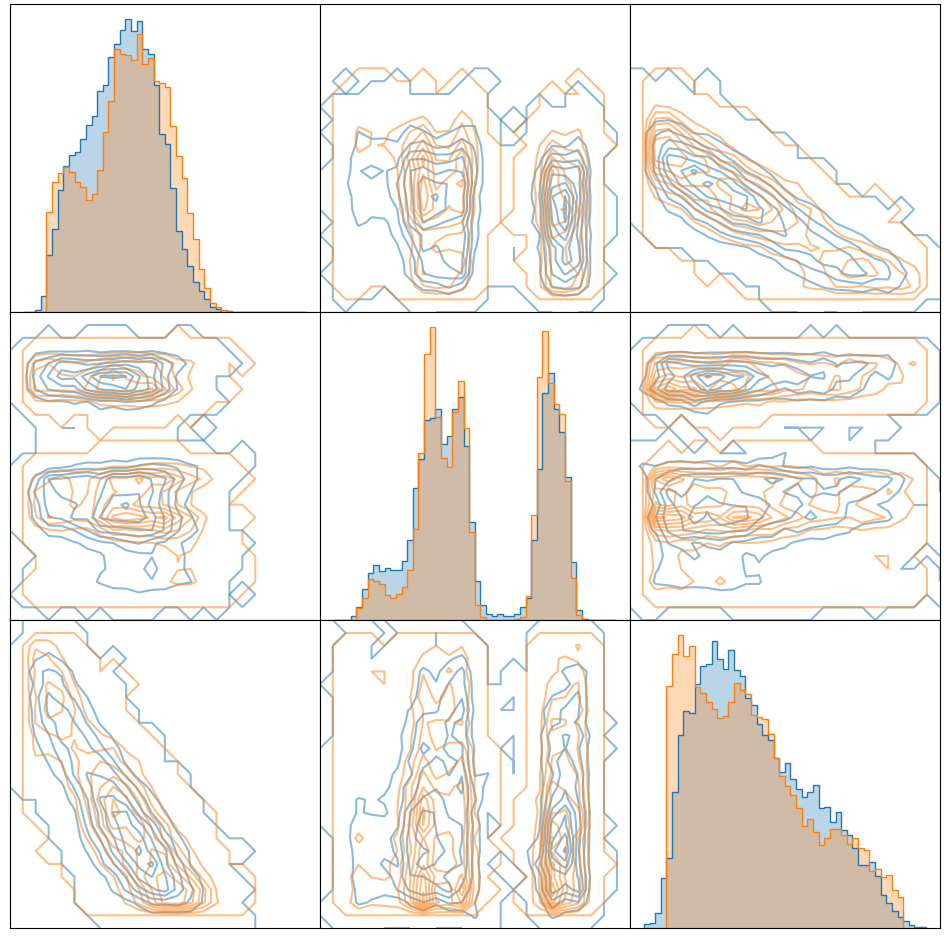}  
\end{subfigure}

\caption{Histograms of the posterior reconstructions using a conditional SNF (left, blue), 
a conditional INN (right, blue) and MCMC (orange) for $2$ different samples from $Y$. On the diagonal we plot the one histograms
of the one-dimensional marginals, on the off-diagonal we plot the potentials of the two 
dimensional marginals.}

\label{fig:scattero}
\end{figure}

\section{Conclusions} \label{sec:conclusions}

We have seen how SNFs can be approached via Markov chains with the advantage
that also distributions without densities can be handled in a sound way using
general Markov kernels and  Radon-Nikodym derivatives.
We showed how the concept can be generalized to  conditional SNFs for sampling from 
posterior distributions in inverse problems and gave first numerical examples.
As future work, we would like to apply conditional SNFs to inverse problems in imaging.
Since the prior distribution $P_X$ has to be known,
one possibility could be to learn in a first step the prior distribution by an INN 
as e.g.\ in \cite{CBDJ2019,DSB2017} or from a patch distribution \cite{HHR2021} and to tackle 
the inverse problem afterwards via a conditional SNF.
It would also be interesting to incorporate other layers as variational autoencoders
or diffusion flows 
and to investigate 
the interplay of both losses of the convex combination \cite{ZZH2021}.

\appendix

\section{Invertible Neural Networks} \label{app:INN}
In this paper, we focus on diffeomorphisms $\mathcal T_t$ determined by 
an INN with the architecture proposed in \cite{AKRK2019}. 
In the following , we skip the index in the description of the INN $\mathcal T_t$,
but keep in mind that different INNs, i.e. different parameters, are learned for different Markov chain indices $t$.
Our INN $\mathcal T = \mathcal T(\cdot; \theta)$ with parameters $\theta$ is a composition
\begin{align} \label{eq_net}
  \mathcal T =  T_{L} \circ P_{L} \circ \dots \circ  T_{1} \circ P_{1},
\end{align}
of permutation matrices $P_\ell$ and invertible mappings $T_\ell$ of the form
\begin{equation*}  \label{eq_DefBlock}
  T_{\ell} (\xi_1,\xi_2) 
  = (x_1,x_2) 
  \coloneqq \left(\xi_1 \, \mathrm{e}^{s_{\ell,2}(\xi_2)} + t_{\ell,2}(\xi_2),\, 
          \xi_2 \, \mathrm{e}^{s_{\ell,1}(x_1)}   + t_{\ell,1}(x_1)   \right)
\end{equation*}
for some splitting $(\xi_1,\xi_2) \in \R^{d}$ with $\xi_i \in \R^{d_i}$, $i=1,2$.  Here
$s_{\ell,2}, t_{\ell,2}: \R^{d_2} \to \R^{d_1}$ 
and 
$s_{\ell,1}, t_{\ell,1}: \R^{d_1} \to \R^{d_2}$ 
are ordinary feed-forward neural networks. 
The parameters $\theta$ of
$\mathcal T(\cdot;\theta)$ are specified by the parameters of these subnetworks.
The inverse of the network layers $T_l$ is analytically given by
\begin{equation*}\label{eq_DefInvBlock}
    T_\ell^{-1}(x_1,x_2) 
    = (\xi_1,\xi_2) 
		\coloneqq \left( \big(x_1 - t_{\ell,2}(\xi_2) \big) \,\mathrm{e}^{-s_{\ell,2}(\xi_2)},\, 
             \big(x_2 - t_{\ell,1}(x_1)   \big) \,\mathrm{e}^{-s_{\ell,1}(x_1)} \right)
\end{equation*}
and does not require an inversion of the feed-forward subnetworks. Hence the whole
map $\mathcal T$ is invertible and allows for a fast evaluation of both forward and
inverse map.

In our loss function, we will need the log-determinant of $\mathcal T$. 
Fortunately this can be simply computed by the following  considerations:
since $T_\ell = T_{2,\ell} \circ T_{1,\ell}$ with 
\begin{align*}
    T_{1,\ell}(\xi_1,\xi_2) = (x_1,\xi_2) 
		&\coloneqq
    \left(
		\xi_1\mathrm{e}^{s_{\ell,2}(\xi_2)} + t_{\ell,2} (\xi_2), \xi_2
    \right),
		\\
		T_{2,\ell}(x_1,\xi_2) = (x_1,x_2) 
		&\coloneqq 
		\left(x_1,
    \xi_2\mathrm{e}^{s_{\ell,1}(x_1)} + t_{\ell,1} (x_1) 
		\right), 
\end{align*}
we have 
\begin{align*}
		  \nabla T_{1,\ell}(\xi_1,\xi_2) 
			&= 
		  \begin{pmatrix}
		    \mathrm{diag} \left( \mathrm{e}^{s_{\ell,2}(\xi_2)} \right) 
				& \mathrm{diag} \left( \nabla_{\xi_2} 
				\left( \xi_1 \mathrm{e}^{s_{\ell,2}(\xi_2)} + t_{\ell,2} (\xi_2) \right) 
				\right)\\
		    0 & I_{d_2}
		  \end{pmatrix}
    \end{align*}
    so that 
		$ \det \nabla T_{1,\ell}(\xi_1,\xi_2) =  \prod_{k=1}^{d_1}
    \mathrm{e}^{\left( s_{\ell,2}(\xi_2)\right)_k} $ 
		and similarly for 
		$\nabla T_{2,\ell}$.  
		Applying the chain rule in \eqref{eq_net} and noting that the
    Jacobian of $P_\ell$ is just $P_\ell$ with $|\det P_\ell|=1 $, 
		and that $\det (A B) = \det(A) \det(B)$, we conclude 
    \begin{align*}
		  \log( |\det \left(\nabla \mathcal T(\xi) \right)|)
		  = \sum_{\ell = 1}^L \left( \operatorname{sum}\left(s_{\ell,2} \left( (P_\ell \xi^{\ell} )_2 \right)\right) 
		  + \operatorname{sum}\left(s_{\ell,1}\left( (T_{1,\ell} P_\ell \xi^{\ell} )_1 \right) \right)\right),
    \end{align*}
    where $\operatorname{sum}$ denotes the sum of the components of the respective vector,
    $\xi^{1} \coloneqq \xi$ and $\xi^{\ell} = T_{\ell-1} P_{\ell-1}
    \xi^{\ell-1}$, $\ell = 2,\ldots,L$.

\section{Proof of Lemma~\ref{prop_KL_MC_paths}} \label{app:proofs}
We show the first equation. The second one can be proven analogously 
using that for a Markov chain $(X_0,...,X_T)$ the time-reversal $(X_T,...,X_0)$ 
is again a Markov chain.
For $f$ in \eqref{def_f}, we consider the measure 
$\mu\coloneqq f\, P_{(X_0,...,X_T)}$ 
and show that $\mu=P_{(Y_0,...,Y_T)}$. 
For this purpose, let $A_0\times\cdots\times A_T\subseteq (\R^d)^{T+1}$ be a measurable rectangle. 
Then, using that by definition of a regular conditional distribution it holds 
$P_{(X_T,X_0,...,X_{T-1})}=P_{X_T}\times P_{(X_0,...,X_{T-1})|X_T}$, 
we get 
\begin{align*}
&\mu(A_0\times\cdots\times A_T)
=\int\limits_{A_0\times\cdots\times A_T}\frac{p_{Y_T}(x_T)}{p_{X_T}(x_T)}
\prod_{t=1}^T f_t(x_{t-1},x_{t}) \dx P_{(X_0,...,X_T)}(x_0,...,x_T)\\
&=\int\limits_{A_T}\frac{p_{Y_T}(x_T)}{p_{X_T}(x_T)}\int\limits_{A_0\times\cdots\times A_{T-1}} 
\prod_{t=1}^T f_t(x_{t-1},x_{t}) \dx 
P_{(X_0,...,X_{L-1})|X_T=x_T}(x_0,...,x_{T-1}) \dx P_{X_T}(x_T).
\end{align*}
Using that $\frac{p_{Y_T}}{p_{X_T}}$ is the Radon-Nikodym derivative of $P_{Y_T}$ with respect to $P_{X_T}$, we obtain 
\begin{align*}
&\quad \mu(A_0\times\cdots\times A_T)\\
&=\int\limits_{A_T}\int\limits_{A_0\times\cdots\times A_{T-1}} \prod_{t=1}^T f_t(x_{t-1},x_{t}) 
\dx P_{(X_0,...,X_{T-1})|X_T=x_T}(x_0,...,x_{T-1}) \dx P_{Y_T}(x_T).
\end{align*}
Thus it suffices to prove for $l=1,...,T$ that
\begin{align}
I_l &\coloneqq
\int\limits_{A_0\times\cdots\times A_{l-1}} \prod_{t=1}^l f_t(x_{t-1},x_{t}) \dx P_{(X_0,...,X_{l-1})|X_l=x_l}(x_0,...,x_{l-1})
\nonumber
\\
&=P_{(Y_0,...,Y_{l-1})|Y_l=x_l}(A_0\times\cdots\times A_{l-1}). \label{eq_induction}
\end{align}
since then
\begin{align*}
 \mu(A_0\times\cdots\times A_T)
&=
\int\limits_{A_T} P_{(Y_0,...,Y_{T-1})|Y_T=x_T}(A_0\times\cdots\times A_{T-1}) \dx P_{Y_T}(x_T)
\\
&=P_{(Y_0,...,Y_T)}(A_0\times\cdots\times A_T)
\end{align*}
and we are done.
We show \eqref{eq_induction} by induction. 
Since 
$f_1(\cdot,x_1)$ 
is the Radon-Nikodym derivative $\frac{d P_{Y_{0}|Y_{1}=x_{1}}}{d P_{X_{0}|X_{1}=x_{1}}}$, 
we have for $l=1$ that
\begin{align*}
\int\limits_{A_0} f_1(x_{0},x_{1}) \dx P_{X_0|X_1=x_1}(x_0)
&=\int\limits_{A_0}1 \,
\dx P_{Y_0|Y_1=x_1}(x_0)=P_{Y_0|Y_1=x_1}(A_0).
\end{align*}
Now assume that \eqref{eq_induction} is true for $l-1$. 
Then it holds
\begin{align*}
I_l &=
\int\limits_{A_{l-1}}\int\limits_{A_0\times\cdots\times A_{l-2}} \prod_{t=1}^{l-1} f_t(x_{t-1},x_{t}) 
\dx P_{(X_0,...,X_{l-2})|X_{l-1}=x_{l-1},X_l=x_l}(x_0,...,x_{l-2}) \\
& \quad \cdot f_l(x_{l-1},x_{l}) 
\dx P_{X_{l-1}|X_l=x_l}(x_{l-1}), \nonumber
\end{align*}
and by the Markov property of $(X_T,...,X_0)$ and the definition of $f_l$ further
 \begin{align*}
 I_l
&=\int\limits_{A_{l-1}}\, I_{l-1} \, \dx P_{Y_{l-1}|Y_l=x_l}(x_{l-1})\\
&=\int\limits_{A_{l-1}} P_{(Y_0,...,Y_{l-2})|Y_{l-1}=x_{l-1}}(A_0\times\cdots\times A_{l-2})\dx P_{Y_{l-1}|Y_l=x_l}(x_{l-1})\\
&=\int\limits_{A_{l-1}} P_{(Y_0,...,Y_{l-2})|Y_{l-1}=x_{l-1},Y_{l}=x_l}(A_0\times\cdots\times A_{l-2}) \dx P_{Y_{l-1}|Y_l=x_l}(x_{l-1})\\
&=P_{(Y_0,...,Y_{l-1})|Y_{l}=x_l}(A_0\times\cdots\times A_{l-1}).
\end{align*}
Since the measurable rectangles are a $\cap$-stable generator of $\mathcal B(\R^d)$, this proves the claim.
\hfill $\Box$

\section{Deterministic and Stochastic Layers for Conditional SNFs}\label{sec_cond_layers}
\vspace{.1cm}

\noindent
\textbf{Deterministic layer:} 
Suppose that $X_t\coloneqq \mathcal T_t(Y,X_{t-1})$ for some measurable mapping 
$\mathcal T_t\colon\R^{\tilde d}\times\R^d \to \R^d$ 
such that $\mathcal T_t(y,\cdot)$ is a diffeomorphism for any $y\in\R^{\tilde d}$. 
Then we use the Markov kernels given by
$$\mathcal K_t(y,x,A) \coloneqq \delta_{\mathcal T_t(y,x)}(A), \quad \mathcal R_t(y,x,A) \coloneqq \delta_{\mathcal T_t^{-1}(y;x)}(A),$$
where $\mathcal T_t^{-1}(y;x)$ denotes the inverse of $\mathcal T_t(y,\cdot)$ at point $x$.

Such mappings $\mathcal T$ can be constructed by using conditional INNs which were 
introduced in \cite{AKLBRK2021,ALKRK2019} and are closely related to conditional GANs \cite{MK2018}.
A \textit{conditional INN}
is a neural network $\mathcal T=\mathcal T(\cdot,\cdot;\theta)\colon\R^{\tilde d}\times \R^d\to\R^d$ such that $\mathcal T(y,\cdot)\colon\R^d\to\R^d$ 
is by construction invertible for any $y\in\R^{\tilde d}$. 
In this paper, we use an adaption of the INN architecture \eqref{eq_net} to model a conditional INN similarly as 
in \cite{ALKRK2019}. 
More precisely, the network $\mathcal T(y,\cdot)$ is a composition
\begin{align} \label{eq_cnet}
  \mathcal T(y,\cdot) =  T_L(y,\cdot) \circ P_L \circ \dots \circ  T_1(y,\cdot) \circ P_1,
\end{align}
of permutation matrices $P_l$ and mappings $T_l\colon\R^{\tilde d}\times\R^d\to\R^d$ of the form
\begin{equation}
  \label{eq_cDefBlock}
  T_l(y,\xi_1,\xi_2) 
  = (x_1,x_2) 
  \coloneqq \left(\xi_1 \, \mathrm{e}^{s_{l,2}(y,\xi_2)} + t_{l,2}(y,\xi_2),\, 
          \xi_2 \, \mathrm{e}^{s_{l,1}(y,x_1)}   + t_{l,1}(y,x_1)   \right)
\end{equation}
for some splitting $(\xi_1,\xi_2) \in \R^{d}$ with $\xi_i \in \R^{d_i}$, $i=1,2$.  
As before
$s_{l,2}, t_{l,2}: \R^{\tilde d} \times \R^{d_2} \to \R^{d_1}$ 
and 
$s_{l,1}, t_{l,1}: \R^{\tilde d} \times \R^{d_1} \to \R^{d_2}$ 
are ordinary feed-forward neural networks 
and the parameters $\theta$ of
$\mathcal T(\cdot,\cdot;\theta)$ are specified by the parameters of these subnetworks.
Note that for any $y\in\R^{\tilde d}$ the mapping $T_l(y,\cdot)$ 
is by definition invertible and admits the analytical inverse
\begin{equation*}\label{eq_cDefInvBlock}
(T_l(y,\cdot))^{-1}(x_1,x_2) 
= (\xi_1,\xi_2) 
\coloneqq \left( \big(x_1 - t_{l,2}(y,\xi_2) \big) \,\mathrm{e}^{-s_{l,2}(y,\xi_2)},\, 
\big(x_2 - t_{l,1}(y,x_1)   \big) \,\mathrm{e}^{-s_{l,1}(y,x_1)} \right).
\end{equation*}
We obtain that $\mathcal T(y,\cdot)$ is invertible for any $y\in\R^{\tilde d}$. Further, both forward and inverse map can be 
computed very efficiently.
\vspace{0.4cm}

\noindent
\textbf{Langevin layer:} 
Let $P_t\colon\R^{\tilde d}\times \mathcal B(\R^d)\to[0,1]$ be a Markov kernel 
such that $P_t(y,\cdot)$ has the density $p_t^y\colon\R^d\to\R_{>0}$. 
Further, define $\xi_t\sim\mathcal N(0,I)$ and $u_t^y(x)=-\log(p_t^y(x))$. 
Then 
$$
X_t \coloneqq X_{t-1}-a_1\nabla u^y_t(X_{t-1})+a_2 \xi_t,
$$
with $a_1,a_2>0$ has the transition kernel $\mathcal K_t$  given by
\begin{equation*}\label{eq_langevin_cond}
\mathcal K_t(y,x,A) \coloneqq \mathcal N(A|x-a_1\nabla u_t^y(x),a_2^2 I).
\end{equation*}
Again, we use $\mathcal R_t=\mathcal K_t$ as reverse layer.
\vspace{0.4cm}

\noindent
\textbf{MCMC layer:} 
Let $P_t\colon\R^{\tilde d}\times \mathcal B(\R^d)\to[0,1]$ 
be a Markov kernel such that $P_t(y,\cdot)$ has the density $p_t^y\colon\R^d\to\R_{>0}$.
Further, let $Q_t\colon\R^{\tilde d}\times\R^d\times \mathcal B(\R^d)$ be a Markov kernel, such that
$
Q_t(y,x,\cdot)
$
admits the strictly positive density $q_t^y(\cdot|x)$.
Define by $X'_t$ a random variable independent of $(X_0,...,X_{t-2})$ such that
$$
P_{X_{t-1},X'_t|Y=y}=P_{X_t|Y=y}\times Q_t(y,\cdot,\cdot)
$$
and assume that
$$
X_t\coloneqq 1_{[U,1]} \big( \alpha_t^Y( X_{t-1},X_t') \big) \, X_t'
+
1_{[0,U]}  \big( \alpha_t^Y( X_{t-1}, X_t' ) \big) \, X_{t-1}
$$
where $\alpha_t^y(x,w) \coloneqq \left\{\min(1,\tfrac{p^y_t(w)q_t^y(x|w)}{p^y_t(x)q_t^y(w|x)})\right\}$.
Then the transition kernel $\mathcal K_t$ is given by
\begin{equation}\label{eq_conditional_MCMC_kernel}
\mathcal K_t(y,x,A) \coloneqq \int_A q_t^y(w|x) \alpha_t^y(x,w)
\dx w
+
\delta_x(A) \int_{\R^d} q_t^y(w|x)(1-\alpha_t^y(x,w))) \dx w.
\end{equation}
Note, that for fixed $y\in\R^{\tilde d}$, the kernel $\mathcal K_t(y,\cdot,\cdot)$ 
is the Metropolis-Hastings kernel with respect to the density $p^y_t$ and Markov kernel $Q_t(y,\cdot,\cdot)$.
Analogously to Remark~\ref{rem_MCMC_step_kernels}, we consider in our numerical examples the kenels
\begin{equation}\label{eq_MCMC_step_MH}
Q_t(y,x,\cdot)=\mathcal N(x,\sigma^2 I),\qquad q_t^y(\cdot|x)=\mathcal N(\cdot|x,\sigma^2 I)
\end{equation}
and
\begin{equation}\label{eq_MCMC_step_MALA}
Q_t(y,x,\cdot)=\mathcal N(x-a_1\nabla u_t^y(x),a_2^2 I),\qquad q_t^y(\cdot|x)=\mathcal N(\cdot|x-a_1\nabla u_t^y(x),a_2^2 I).
\end{equation}
As in the non-conditional case we use $\mathcal R_t=\mathcal K_t$ as reverse layer.

\section*{Acknowledgment}
 The funding by the German Research Foundation (DFG) within 
 the projects STE 571/16-1 and 
within the project of the DFG-SPP 2298 ,,Theoretical Foundations of Deep Learning'' 
is gratefully acknowledged.
Many thanks to S. Heidenreich from the Physikalisch-Technische Bundesanstalt (PTB) for
providing the scatterometry data which we used for training the forward model 
and for fruitful discussions on the corresponding example. 
P. H.  thanks J. K\"ohler for helpful discussions.

\bibliographystyle{abbrv}
\bibliography{ref}

\begin{thebibliography}{10}

\bibitem{AFHHSS2021}
A.~Andrle, N.~Farchmin, P.~Hagemann, S.~Heidenreich, V.~Soltwisch, and
  G.~Steidl.
\newblock Invertible neural networks versus {MCMC} for posterior reconstruction
  in grazing incidence {X}-ray fluorescence.
\newblock In A.~Elmoataz, J.~Fadili, Y.~Quéau, J.~Rabin, and L.~Simon,
  editors, {\em Scale Space and Variational Methods}, volume 12679 of {\em
  Lecture Notes in Computer Science}, pages 528--539. Springer, 2021.

\bibitem{arbel2021annealed}
M.~Arbel, A.~G. D.~G. Matthews, and A.~Doucet.
\newblock Annealed flow transport monte carlo.
\newblock {\em ArXiv:2102.07501}, 2021.

\bibitem{AKLBRK2021}
L.~Ardizzone, J.~Kruse, C.~L{\"u}th, N.~Bracher, C.~Rother, and U.~K{\"o}the.
\newblock Conditional invertible neural networks for diverse image-to-image
  translation.
\newblock In T.~S. Z.~Akata, A.~Geiger, editor, {\em Pattern Recognition: 42nd
  DAGM German Conference, DAGM GCPR 2020}, volume 12544 of {\em Lecture Notes
  in Computer Science}, pages 373--387. Springer, 2021.

\bibitem{AKRK2019}
L.~Ardizzone, J.~Kruse, C.~Rother, and U.~K{\"{o}}the.
\newblock Analyzing inverse problems with invertible neural networks.
\newblock In {\em International Conference on Learning Representations}, 2019.

\bibitem{ALKRK2019}
L.~Ardizzone, C.~L\"uth, J.~Kruse, C.~Rother, and U.~K\"othe.
\newblock Guided image generation with conditional invertible neural networks.
\newblock {\em ArXiv:1907.02392}, 2019.

\bibitem{AMOS2019}
S.~Arridge, P.~Maass, O.~\"Oktem, and C.~B. Sch\"onlieb.
\newblock Solving inverse problems using data-driven models.
\newblock {\em Acta Numer.}, 28:1--174, 2019.

\bibitem{BGCDJ2019}
J.~Behrmann, W.~Grathwohl, R.~T. Chen, D.~Duvenaud, and J.-H. Jacobsen.
\newblock Invertible residual networks.
\newblock In {\em International Conference on Machine Learning}, pages
  573--582, 2019.

\bibitem{behrmann2020understanding}
J.~Behrmann, P.~Vicol, K.-C. Wang, R.~Grosse, and J.-H. Jacobsen.
\newblock Understanding and mitigating exploding inverses in invertible neural
  networks.
\newblock {\em ArXiv:2006.09347}, 2020.

\bibitem{BFPS17}
R.~Bergmann, J.~H. Fitschen, J.~Persch, and G.~Steidl.
\newblock Iterative multiplicative filters for data labeling.
\newblock {\em Int. J. Comput. Vis.}, 123(3):123--145, 2017.

\bibitem{Chenetal2017}
C.~Chen, C.~Li, L.~Chen, W.~Wang, Y.~Pu, and L.~Carin.
\newblock Continuous-time flows for efficient inference and density estimation.
\newblock {\em ArXiv:1709.01179}, 2017.

\bibitem{CBDJ2019}
R.~T.~Q. Chen, J.~Behrmann, D.~K. Duvenaud, and J.-H. Jacobsen.
\newblock Residual flows for invertible generative modeling.
\newblock In {\em Advances in Neural Information Processing Systems},
  volume~32. Curran Associates, Inc., 2019.

\bibitem{CCDD19}
R.~Cornish, A.~L. Caterini, G.~Deligiannidis, and A.~Doucet.
\newblock Relaxing bijectivity constraints with continuously indexed
  normalising flows.
\newblock {\em ArXiv:1909.13833}, 2019.

\bibitem{CTA2019}
N.~De~Cao, I.~Titov, and W.~Aziz.
\newblock Block neural autoregressive flow.
\newblock {\em ArXiv:1904.04676}, 2019.

\bibitem{DSB2017}
L.~Dinh, J.~Sohl{-}Dickstein, and S.~Bengio.
\newblock Density estimation using real {NVP}.
\newblock In {\em International Conference on Learning Representations}, 2017.

\bibitem{DBMP2019}
C.~Durkan, A.~Bekasov, I.~Murray, and G.~Papamakarios.
\newblock Neural spline flows.
\newblock {\em Advances in Neural Information Processing Systems}, 2019.

\bibitem{Falorsietal2018}
L.~Falorsi, P.~de~Haan, T.~R. Davidson, W.~De~Cao, N., P.~M., Forré, and T.~S.
  Cohen.
\newblock Explorations in homeomorphic variational auto-encoding.
\newblock {\em ArXiv:1807.04689}, 2018.

\bibitem{FHDF2019}
L.~Falorsi, P.~de~Haan, T.~R. Davidson, and P.~Forr\'e.
\newblock Reparameterizing distributions on {L}ie groups.
\newblock {\em ArXiv:1903.02958}, 2019.

\bibitem{POT}
R.~Flamary, N.~Courty, A.~Gramfort, M.~Z. Alaya, A.~Boisbunon, S.~Chambon,
  L.~Chapel, A.~Corenflos, K.~Fatras, N.~Fournier, L.~Gautheron, N.~T. Gayraud,
  H.~Janati, A.~Rakotomamonjy, I.~Redko, A.~Rolet, A.~Schutz, V.~Seguy, D.~J.
  Sutherland, R.~Tavenard, A.~Tong, and T.~Vayer.
\newblock {POT}: Python optimal transport.
\newblock {\em J. Mach. Learn. Res.}, 22(78):1--8, 2021.

\bibitem{GC2011}
M.~Girolami and B.~Calderhead.
\newblock Riemann manifold {L}angevin and {H}amiltonian {M}onte {C}arlo
  methods.
\newblock {\em J. R. Stat. Soc.: Series B (Statistical Methodology)},
  73(2):123--214, 2011.

\bibitem{GFO2017}
D.~Grana, T.~Fjeldstad, and H.~Omre.
\newblock Bayesian {G}aussian mixture linear inversion for geophysical inverse
  problems.
\newblock {\em Math. Geosci.}, 49(4):493--515, 2017.

\bibitem{HN2021}
P.~L. Hagemann and S.~Neumayer.
\newblock Stabilizing invertible neural networks using mixture models.
\newblock {\em Inverse Probl.}, 2021.

\bibitem{HZRS2016}
K.~He, X.~Zhang, S.~Ren, and J.~Sun.
\newblock Deep residual learning for image recognition.
\newblock In {\em Proceedings of the IEEE Conference on Computer Vision and
  Pattern Recognition}, pages 770--778, 2016.

\bibitem{heidenreich2015bayesian}
S.~Heidenreich, H.~Gross, and M.~B\"ar.
\newblock Bayesian approach to the statistical inverse problem of
  scatterometry: Comparison of three surrogate models.
\newblock {\em Int. J. Uncertain. Quantif.}, 5(6), 2015.

\bibitem{scatter}
S.~{Heidenreich}, H.~{Gross}, and M.~{B{\"a}r}.
\newblock {Bayesian approach to determine critical dimensions from
  scatterometric measurements}.
\newblock {\em Metrologia}, 55(6):S201, Dec. 2018.

\bibitem{HHR2021}
J.~Hertrich, A.~Houdard, and C.~Redenbach.
\newblock Wasserstein patch prior for image superresolution.
\newblock {\em ArXiv:2109.12880}, 2021.

\bibitem{huang2018neural}
C.-W. Huang, D.~Krueger, A.~Lacoste, and A.~Courville.
\newblock Neural autoregressive flows.
\newblock In {\em International Conference on Machine Learning}, pages
  2078--2087, 2018.

\bibitem{JKYB20}
P.~Jaini, I.~Kobyzev, Y.~Yu, and M.~Brubaker.
\newblock Tails of lipschitz triangular flows.
\newblock {\em ArXiv:1907.04481}, 2019.

\bibitem{KB2015}
D.~P. Kingma and J.~Ba.
\newblock Adam: {A} method for stochastic optimization.
\newblock In {\em International Conference on Learning Representations}, 2015.

\bibitem{kingma2018glow}
D.~P. Kingma and P.~Dhariwal.
\newblock Glow: Generative flow with invertible 1x1 convolutions.
\newblock {\em ArXiv:1807.03039}, 2018.

\bibitem{HINT}
J.~Kruse, G.~Detommaso, U.~K{\"{o}}the, and R.~Scheichl.
\newblock {HINT:} hierarchical invertible neural transport for density
  estimation and bayesian inference.
\newblock In {\em {AAAI} Conference on Artificial Intelligence}, pages
  8191--8199, 2021.

\bibitem{LeGall}
J.-F. Le~Gall.
\newblock {\em Brownian motion, martingales, and stochastic calculus}, volume
  274 of {\em Graduate Texts in Mathematics}.
\newblock Springer, [Cham], 2016.

\bibitem{minka2005divergence}
T.~Minka.
\newblock Divergence measures and message passing.
\newblock Technical report, Technical report, Microsoft Research, 2005.

\bibitem{MK2018}
T.~Miyato and M.~Koyama.
\newblock c{GAN}s with projection discriminator.
\newblock In {\em International Conference on Learning Representations}, 2018.

\bibitem{MMRGN2018}
T.~M\"uller, B.~McWilliams, F.~Rousselle, M.~Gross, and J.~Nov\'ak.
\newblock Neural importance sampling.
\newblock {\em ArXiv:1808.03856}, 2018.

\bibitem{NJHWW2020}
D.~Nielsen, P.~Jaini, E.~Hoogeboom, O.~Winther, and M.~Welling.
\newblock {SurVAE Flows}: surjections to bridge the gap between {VAE}s and
  flows.
\newblock {\em NeurIPS}, 2020.

\bibitem{NCMC2011}
J.~P. Nilmeier, G.~E. Crooks, D.~D.~L. Minh, and J.~Chodera.
\newblock Nonequilibrium candidate {M}onte {C}arlo is an efficient tool for
  equilibrium simulation.
\newblock {\em Proceedings of the National Academy of Sciences},
  108:1009–1018, 2011.

\bibitem{OJMBDW2020}
G.~Ongie, A.~Jalal, C.~A. Metzler, R.~G. Baraniuk, A.~G. Dimakis, and
  R.~Willett.
\newblock Deep learning techniques for inverse problems in imaging.
\newblock {\em {IEEE} J. Sel. Areas Inf. Theory}, 1(1):39--56, 2020.

\bibitem{PPM2017}
G.~Papamakarios, T.~Pavlakou, , and I.~Murray.
\newblock Masked autoregressive flow for density estimation.
\newblock {\em Advances in Neural Information Processing Systems}, page
  2338–2347, 2017.

\bibitem{pmlr-v37-rezende15}
D.~Rezende and S.~Mohamed.
\newblock Variational inference with normalizing flows.
\newblock In {\em International Conference on Machine Learning}, volume~37 of
  {\em Proceedings of Machine Learning Research}, pages 1530--1538, Lille,
  France, 2015. PMLR.

\bibitem{RM2015}
D.~J. Rezende and S.~Mohamed.
\newblock Variational inference with normalizing flows.
\newblock {\em ArXiv:1505.05770}, 2015.

\bibitem{robertsrosenthalMCMC}
G.~O. Roberts and J.~S. Rosenthal.
\newblock {General state space Markov chains and MCMC algorithms}.
\newblock {\em Probab. Surv.}, 1:20 -- 71, 2004.

\bibitem{RT1996}
G.~O. Roberts and R.~L. Tweedie.
\newblock Exponential convergence of {L}angevin distributions and their
  discrete approximations.
\newblock {\em Bernoulli}, 2(4):341--363, 1996.

\bibitem{RH2021}
L.~Ruthotto and E.~Haber.
\newblock An introduction to deep generative modeling.
\newblock {\em DMV Mitteilungen}, 44(3):1--24, 2021.

\bibitem{SWMG2015}
J.~Sohl-Dickstein, E.~A. Weiss, N.~Maheswaranathan, and S.~Ganguli.
\newblock Deep unsupervised learning using nonequilibrium thermodynamics.
\newblock {\em ArXiv:1503.03585}, 2015.

\bibitem{anicemc}
J.~Song, S.~Zhao, and S.~Ermon.
\newblock {A-NICE-MC}: Adversarial training for {MCMC}.
\newblock {\em ArXiv:1706.07561}, 2018.

\bibitem{doucetthin}
A.~Thin, N.~Kotelevskii, A.~Doucet, A.~Durmus, E.~Moulines, and M.~Panov.
\newblock Monte {C}arlo variational auto-encoders.
\newblock In {\em Proceedings of the 38th International Conference on Machine
  Learning}, volume 139 of {\em Proceedings of Machine Learning Research},
  pages 10247--10257, 2021.

\bibitem{TLA2020}
D.~Tsvetkov, L.~Hristov, and R.~Angelova-Slavova.
\newblock On the convergence of the {M}etropolis-{H}astings {M}arkov chains.
\newblock {\em ArXiv:1302.0654v4}, 2020.

\bibitem{langevin_welling}
M.~Welling and Y.~W. Teh.
\newblock Bayesian learning via stochastic gradient langevin dynamics.
\newblock In {\em International Conference on Machine Learning}, page
  681–688, 2011.

\bibitem{WLD2021}
J.~Whang, E.~M. Lindgren, and A.~G. Dimakis.
\newblock Composing normalizing flows for inverse problems.
\newblock In {\em International Conference on Machine Learning}, 2021.

\bibitem{WKN2020}
H.~Wu, J.~K{\"{o}}hler, and F.~No{\'{e}}.
\newblock Stochastic normalizing flows.
\newblock In {\em Advances in Neural Information Processing Systems}, 2020.

\bibitem{ZZH2021}
S.~Zhang, P.~Zhang, and T.~Y. Hou.
\newblock Multiscale invertible generative networks for high-dimensional
  bayesian inference.
\newblock {\em ArXiv:2105.05489}, 2021.

\end{thebibliography}
\end{document}